\documentclass{article}
\usepackage{geometry}

\usepackage[utf8]{inputenc} 
\usepackage[T1]{fontenc}    
\usepackage[pagebackref=true]{hyperref}       
\usepackage{url}            
\usepackage{booktabs}       
\usepackage{amsfonts}       
\usepackage{nicefrac}       
\usepackage{microtype}      
\usepackage{xcolor}         
\usepackage{amsmath}
\usepackage{algorithm}
\usepackage{algorithmic}
\usepackage{bbm}
\usepackage{subfig}
\usepackage{graphicx}
\usepackage{mathtools}
\usepackage{pdflscape}
\usepackage{listings}
\usepackage{caption}
\usepackage{amsthm}
\usepackage{enumitem}
\usepackage{authblk}
\usepackage{natbib}

\renewcommand*\backref[1]{\ifx#1\relax \else (Cited on #1) \fi}

\theoremstyle{plain}
\newtheorem{theorem}{Theorem}[section]

\newtheorem{lemma}[theorem]{Lemma}

\theoremstyle{definition}

\newtheorem{assumption}[theorem]{Assumption}
\theoremstyle{remark}
\newtheorem{remark}[theorem]{Remark}

\newcommand{\E}{\mathbb{E}}
\newcommand{\R}{\mathbb{R}}
\newcommand{\N}{\mathbb{N}}
\newcommand{\bP}{\mathbb{P}}
\newcommand{\bQ}{\mathbb{Q}}
\newcommand{\1}{\mathbbm{1}}

\newcommand{\w}{\mathbf{w}}

\newcommand{\cS}{\mathcal{S}}
\newcommand{\cX}{\mathcal{X}}

\newcommand{\ip}[1] {\left\langle #1 \right\rangle }

\newcommand{\cR}{\mathcal{R}}
\newcommand{\cD}{\mathcal{D}}

\newcommand{\cI}{\mathcal{I}}
\newcommand{\cH}{\mathcal{H}}
\newcommand{\cA}{\mathcal{A}}

\newcommand{\cO}{\mathcal{O}}

\newcommand{\bfx}{\mathbf{x}}
\newcommand{\bfy}{\mathbf{y}}

\newcommand{\bfz}{\mathbf{z}}
\newcommand{\bfw}{\mathbf{w}}
\newcommand{\Oc}{\mathcal{O}}
\newcommand{\tPhi}{\tilde{\Phi}}

\DeclareMathOperator*{\argmin}{arg\,min}
\DeclareMathOperator*{\argmax}{arg\,max}

\title{Collaborative Learning in Kernel-based Bandits for Distributed Users}

\author[$\ddagger$]{Sudeep Salgia}
\author[*]{Sattar Vakili}
\author[$\ddagger$]{Qing Zhao}
\affil[$\ddagger$]{School of Electrical \& Computer Engineering, Cornell University, Ithaca, NY, \emph{\{ss3827,qz16\}@cornell.edu} }
\affil[*]{MediaTek Research, UK, \emph{sattar.vakili@mtkresearch.com}}

\begin{document}

\maketitle

\begin{abstract}
We study collaborative learning among distributed clients facilitated by a central server. Each client is interested in maximizing a personalized objective function that is a weighted sum of its local objective and a global objective. Each client has direct access to random bandit feedback on its local objective, but only has a partial view of the global objective and relies on information exchange with other clients for collaborative learning. We adopt the kernel-based bandit framework where the objective functions belong to a reproducing kernel Hilbert space. We propose an algorithm based on surrogate Gaussian process (GP) models and establish its order-optimal regret performance (up to polylogarithmic factors). We also show that the sparse approximations of the GP models can be employed to reduce the communication overhead across clients.
\end{abstract}

\section{Introduction}

 \subsection{Kernel-based Bandit Problems}
We consider zeroth-order stochastic optimization where the objective function is unknown but assumed to live in a Reproducing Kernel Hilbert Space (RKHS) associated with a known kernel. A learner sequentially chooses points to query and receives noisy feedback on the function values at the query points. The goal is to converge quickly to the optimizer of the underlying unknown function. A typical application is learning a classifier or predictor 
based on random instances (e.g., text prediction based on past typing instances). The stochastic objective function is the \emph{expected} reward representing the performance of the predictors. The objective function is unknown due to the unknown probabilistic model of the dependency of the next word on the past. The learner can try a predictor (i.e., the current query point) on the current instance and observes the associated reward, which is a random realization of the objective function value at the queried point.  \\

The above zeroth-order stochastic optimization problem can be equivalently viewed as a continuum-armed kernelized-bandit problem~\citep{Kanagawa2018}. The RKHS model allows for a greater expressive power for representing a broad family of objective functions. In particular, it is known that the RKHS of typical kernels, such as Mat\'ern family of kernels, can approximate almost all continuous functions on compact subsets of $\R^d$~\citep{Srinivas2010}.
Under a centralized setting where all data is available at a single decision maker, several algorithms have been proposed, including  UCB-based algorithms~\citep{Srinivas2010, Chowdhury2017}, batched pure exploration~\citep{li2022gaussian}, and tree-based domain shrinking~\citep{Salgia2020GPThreDS}. \\

In many applications, we encounter stochastic optimization in a distributed setting with multiple clients. For example, in a federated learning setting of the text prediction problem, multiple mobile users aim to learn a text predictor by leveraging each other's local data but without sharing their raw data~\citep{Li2020Survey}. Another example, more specific to the kernelized bandit setting, is hyperparameter tuning in federated learning~\citep{Dai2020FBO}. Due to the large expressive power of RKHS functions, hyperparameter tuning in centralized settings is often carried out by modelling the accuracy of the underlying neural network as an unknown function of the hyperparameters and belonging to an RKHS.
However, when a model is being learned by a collective effort of multiple clients, the hyperparameters of the model often also need to be tuned collaboratively using the information from all clients, resulting in a distributed stochastic optimization problem of a RKHS function. An interesting application of distributed optimization of RKHS functions is in the analysis of the collaborative training of neural nets using the recent theory of Neural Tangent Kernel~\citep{jacot2018neural}. \\



Comparing with the centralized setting, distributed stochastic optimization gives rise to two new challenges. First, the probabilistic model underlying local data generation may be different across clients, resulting in different local objective functions and hence the pursuit of different optimizers at distributed clients. Consider for example the text prediction problem. The probabilistic dependency of the next word on the past typed words may exhibit discrepancies caused by  geographical or cultural differences across users. This drift in data distribution calls for learning algorithms that strike a balance between data sharing and targeted local performance. In particular, how heterogeneous clients can learn collectively by leveraging each other's data while at the same time converging to an optimizer that takes into account each client's local data model. 
The prevailing approach where all clients aim to optimize a common global objective function given by a weighted sum of each client's local objectives with prefixed weight operates at one extreme of the spectrum by focusing solely on reaping the benefit of data sharing while ignoring each client's local performance.  
The second challenge is the communication overhead associated with information exchange across clients. While communication is necessary for collaborative learning, low communication cost is of practical importance in many applications.
These challenges have not been adequately addressed in the literature, especially within the kernel-based distributed bandit framework (see Sec.~\ref{sub:related_work} for a detailed discussion of related work).  \\


\subsection{Main Contributions}

Consider a distributed setting with $K$ clients connected to a central server that facilitates information exchange, a de facto setting for training large-scale machine learning models. Each client can sequentially query its local objective function (i.e., the expected reward with respect to its local data model) and observe an unbiased noisy estimate of the function value at the chosen query point. To handle heterogeneity in local data models, we adopt the personalization framework, wherein the objective of each client is to learn a mixture of their local function and the global function, where the global objective function is an average of all local functions. The personalization framework strikes a balance between the generalization capabilities of the global function and targeted local learning tailored towards each client's local model. By allowing the full range of the mixture weights (i.e., the personalization parameters), we address the full spectrum in terms of the balance between data sharing and targeted local performance.   \\


Within this personalization framework, each client is interested in maximizing its own \emph{personalized objective function}, which differs from those of other clients. The challenge here is that the personalized objectives ${f}_k$ at each user $k$ involves other users' local functions $\{h_j\}_{j\neq k}$ which cannot be observed by user $k$. In other words, the local observations at each end user $k$ provide only a partial view of the personalized objective ${f}_k$ it aims to optimize. This necessitates collaborative learning across users in order to achieve a sublinear regret order. The required user collaboration has two aspects: information sharing and collaborative exploration. Given that $h_j$ can only be queried at user $j$, it is easy to see that information sharing is necessary for each user $k$ to learn its personalized objective ${f}_k$. The need for collaborative exploration is a more nuanced and complex issue. Consider the extreme case where every user's local observations are made available to all other users. If user $k$ simply uses a centralized learning algorithm to optimize its own personalized objective function ${f}_k$, its local model may concentrate too quickly 
around its own optimal point, resulting in insufficient exploration of $h_k$ around optimal points of the personalized local models of other users, especially for a user whose optimal point is far away from that of user $k$. More specifically, even though each user aims to converge to different optimal models, collaborative exploration--in addition to information sharing---is necessary due to the coupling across users' personalized objective functions. This adds another layer of nuance to the exploration-exploitation trade-off as the exploration at every user needs to serve the goal of greater good at the network level instead of solely being sufficient for local exploitation.  \\


We propose a new learning algorithm for kernelized distributed bandits with probabilistically heterogeneous clients. Referred to as Collaborative Exploration with Personalized Exploitation (CEPE), this algorithm is built on the key structure of interleaving exploration epochs for the collaborative learning of the global function with exploitation epochs aiming to maximize individual local performance. 
This interleaving exploration-exploitation structure not only addresses the trade-off between of data leveraging and targeted local performance, but also effectively controls the information exchange across clients for communication efficiency through a tandem design of the accompanying communication protocol. The lengths of the exploration and exploitation epochs are carefully controlled to balance learning efficiency and communication overhead associated with information exchange. \\

We analyze the performance of the proposed algorithm and show that its regret is of $\tilde{\cO}(KT^{\frac{2}{3- \kappa}})$\footnote{$\tilde{\cO}(\cdot)$ hides the polylogarthmic factors.} and its overall communication cost is of $\cO(T^{\frac{2}{3- \kappa}})$, where $\kappa \in (0,1]$ is a parameter that depends on the smoothness of the kernel.  We further establish a lower bound on achievable regret within the kernel-based collaborative learning framework for the class of Squared Exponential and Matern kernels, which are arguably the most widely used kernels in practice. The regret lower bound matches the order of the regret incurred by CEPE up to a poly-logarithmic factor, establishing the optimality of the proposed algorithm. Empirical studies demonstrate the effectiveness of the CEPE against several baseline algorithms, bolstering our theoretical results. \\

To further reduce the communication cost, we propose a variant of CEPE which employs sparse approximations of the surrogate Gaussian Process modelling used in the query and update rules of CEPE. Referred to as S-CEPE, this algorithm reduces the communication cost to $\cO(T^{\frac{2\kappa}{3- \kappa}})$ while preserving the regret guarantee of CEPE. Numerical examples demonstrate effectiveness of S-CEPE in practice as S-CEPE is shown to offer a $14$-fold reduction in communication cost while sacrificing regret by a factor of less than $2$ when compared to CEPE. \\

\subsection{Related Work}
\label{sub:related_work}

Kernel-based bandit problem in the centralized setting has been extensively studied in the literature~\citep{Chowdhury2017,Srinivas2010,Valko2013a, Hensman2013, Titsias2009, Calandriello2017, Seeger2003}. The setup considered in these studies, where all the data is available at a central server, is inherently different from the distributed setup considered in this work. Please refer to  Remark~\ref{remark:singlekernel} for additional comparison with these results. \\

Various studies have considered the discrete multi-armed bandit problem and linear bandit problem in the distributed setting.
\cite{Shi2021FMAB} and~\cite{Shi2021Personalization} consider the multi-armed bandit (MAB) problem in a distributed setup with homogeneous and heterogeneous client reward distributions respectively. \cite{Li2020} considered the same problem with distributionally identical clients, with a focus on ensuring privacy.~\cite{Hillel2013} consider the pure exploration problem for multi-armed bandits connected over a network. Other representative works in the MAB setting include~\citep{Liu2010DistributedMAB, Shahrampour2017, Landgren2017, Sankararaman2019}.~\cite{Wang2019distributedBandit} proposed the DELB algorithm for the distributed linear bandit setting with focus on communication efficiency.~\cite{Dubey2020} proposed a privacy preserving algorithm for distributed linear bandits with clients connected via various network topologies. Other studies exploring linear bandits in a distributed setup include~\citep{Korda2016, Huang2021, Ghosh2021, Amani2022, Salgia2022LinearBandits}. The kernel-based bandit setting is arguably more challenging than the classical MAB or the linear bandit setting. \\

The problem of kernel-based bandits in collaborative setups has not received sufficient attention. There exist only a handful studies.~\cite{Du2021CoPEKB} considered the problem of pure exploration in kernel-based collaborative learning over a finite action set. Another work in this direction is by~\cite{Li2022KernelContextBandits} where the authors propose a communication-efficient algorithm with a communication cost of $\cO(\gamma_T^{3})$.~\cite{Dubey2020KernelBandits} consider distributed kernel bandits over a graph where they consider additional kernel based modelling to measure task similarity among the different objective functions of the clients. However, their proposed algorithm suffers from a communication cost that grows linearly with the time horizon. In this work, we consider the more challenging continuum-armed setup with a focus on minimizing cumulative regret as opposed to simple regret. Moreover, the personalization framework considered in this work is different from the other studies in kernel-based collaborative learning. Furthermore, the S-CEPE algorithm proposed in this work incurs a communication cost of $\cO(T^{\frac{2\kappa}{3- \kappa}})$, improving upon the $\cO(T^{3\kappa})$ bound obtained for the algorithm in~\cite{Li2022KernelContextBandits}. \\

The problem of distributed first-order stochastic convex optimization, often referred to as Federated Learning (FL)~\citep{mcmahan2017communication}, is yet another direction of related work that has been gaining a lot of attention in the recent times. See~\cite{abdulrahman2021survey, Kairouz2021} for a survey of recent advances. Several works in FL have also considered the personalization framework~\citep{Smith2017, Jiang2019, Wang2019, Deng2020, fallah2020personalized, Hanzely2020, Kulkarni2020, Mansour2020} as considered in this work. Despite some similarities in the setup, the zeroth-order kernel-based setup considered in this work requires very different tools and techniques for algorithm design and analysis as compared to the first-order stochastic convex optimization problem considered in FL. 

\section{Problem Formulation and Preliminaries}
\label{sec:problem_formulation}

In this section, we present the problem formulation followed by preliminaries on Gaussian Processes (GPs), an important tool in the design of the proposed algorithm.

\subsection{Problem Formulation}
\label{sub:fl_personalization}

We consider a collaborative learning framework with a star topology consisting of a central server and $K>1$ clients. Each client $i \in \{1, 2 ,\dots, K\}$ is associated with a local \emph{observation} function, $h_i : \cX \to \R$, which it can access by querying any point $x$ in the domain $\cX \subset \R^d$ and consequently receiving a noisy evaluation $y = h_{i}(x) + \epsilon$, where $\epsilon$ is the noise term. These observation functions $h_i$ are known to live in the Reproducing Kernel Hilbert Space (RKHS) associated with a known positive definite kernel $k:\cX\times\cX\rightarrow \R$. We make the following assumptions on the observation functions and noise, commonly adopted in kernel-based learning~\citep{Srinivas2010, Chowdhury2017}.



\begin{assumption}\label{assumption1}
The RKHS norm of $h_i$ is bounded by a known constant $B_i$: $\|h_i\|_{H_k}\le B_i$, for each $i \in \{1,2,\dots, K\}$.
\end{assumption}

\begin{assumption}
The noise terms $\epsilon$ are assumed to be zero mean $R$-sub-Gaussian variables, i.e., $\E \left[ e^{\zeta \epsilon} \right] \leq \exp \left( \zeta^2 R^2/2 \right)$ for all $\zeta \in \R$, that are distributed independently across queries and clients. 
\label{assumption:sub_gaussian}
\end{assumption}

\begin{assumption}
For each given $n \in \N$ and $f \in H_k$ with $\|f\|_{H_k} \leq B$, there exists a discretization $\cD_n$
of $\cX$ such that $|f(x) - f([x]_{\cD_n})| \leq 1/\sqrt{n}$, where $[x]_{\cD_n} = \argmin_{y \in \cD_n} \|x - y\|_2$ is the closest point in $\cD_n$ to $x$ in terms of $\ell_2$ distance and $|\cD_n| \leq C B^d n^{d/2}$, where $C$ is a constant independent of $n$ and $B$.
\label{assumption:discretization}
\end{assumption}

\subsubsection{Personalized Reward Function}

In a typical collaborative learning framework, the clients work together to optimize a common objective function, referred to as the \emph{global} reward function, which is defined as,
\begin{align}
    g(\cdot) := \frac{1}{K} \sum_{i = 1}^K h_i(\cdot).
\end{align}
For example,~\cite{Shi2021FMAB} considered the finite-armed bandit problem with the above mentioned objective. In a more general case, the weights $1/K$ can be replaced by any weight distribution over the $K$ clients. In typical collaborative learning problems, the motivation behind optimizing such a \emph{global} function is to leverage information across clients. The availability of additional data facilitated by the participation of more clients helps in learning a better model that generalizes well. In the context of this work, the argument of $g$ corresponds to the parameters of the model and consequently, the maximizer of $g$ corresponds to the desired optimal parameters. \\

While collaboratively optimizing $g$ can help leverage information across clients, optimizing the global objective $g$ may not always faithfully represent the interests or desires of the clients, especially in the scenario of statistically heterogeneous clients. In contrast, it can be detrimental for a client to only stick to model obtained using local data and forgo the additional generalization capabilities offered by collaboration. Following~\cite{Deng2020, Hanzely2020, Shi2021Personalization}, we consider a collaborative learning model with personalization, where each client can choose to trade-off between the generalization capabilities of the global reward function and a locally focused reward function that aims to fit local data only.
In particular, we consider a personalized \emph{reward} function for each client $i$ given by,
\begin{align}
    f_i(\cdot) = \alpha_i h_i(\cdot) + (1 - \alpha_i)g(\cdot),\label{eqn:personalised_reward_function}
\end{align}
which is a linear combination of the global reward and the local observation functions. 
Here, the personalization parameter $\alpha_i\in (0,1)$ captures the level of personalization for each client $i$. A choice of $\alpha_i \to 0$ corresponds to adopting a fully collaborative approach. Increasing the value of $\alpha_i$ reduces the impact of collaboration and as $\alpha_i \to 1$, the client moves towards adopting a model that solely targets local observations, forgoing all benefits of collaboration. In this work, we focus on the cases $\alpha_i \in (0,1)$\footnote{This condition can be relaxed to $\max_i \alpha_i > 0$ and $\min_i \alpha_i < 1$.} which corresponds to the scenario where clients have different objective functions and need to collaborate while having different personalized objectives. We would like to point out that this personalized setting is more challenging than the scenario where all clients collaboratively optimize a \emph{common} global objective (all $\alpha_i =0$) or all of them are interested in optimizing their local observation functions (all $\alpha_i = 1$).
For clarity of terminology, we refer to $g$ as the global reward, $h_i$ as the local observation, and $f_i$ as the personalized reward for client $i$.

\subsubsection{Communication Protocol and Cost}
\label{sub:communication_cost_definition}


Each personalized reward $f_i$ depends on the observation functions $h_j$ of other clients $j \neq i$ which are cannot be accessed by client $i$. The clients thus need to share information about their local observations with each other, in order to make learning $f_i$ feasible and incur a sublinear regret. In our setup, the clients can communicate only with central server and cannot communicate with other clients, a typical setting in networks with a star topology. At each time $t$, each client $i$ is allowed to send and receive a message to and from the server. For simplicity, we assume that all the communication is synchronized, a commonly adopted assumption in collaborative learning frameworks. We assign a unit cost to each communication of a scalar from the client to other clients through the server. 

\subsubsection{The learning objective}

A collaborative learning policy $\pi=\{\pi_{i,t}\}_{t\ge 1, \, i=1,2,\dots,K}$ specifies, for each client~$i$, which point $x_{i,t}$ to query at each time $t$ based on available information. 
The performance of $\pi$ is measured in terms of total cumulative regret summed over all clients and over a learning horizon of length $T$. Specifically, 
\begin{align}
    R_{\pi}(T) = \sum_{i = 1}^K \sum_{t = 1}^T (f_i(x^*_i) - f_i(x_{i,t})),
    \label{eqn:regret_definition}
\end{align}
where $x_i^* = \argmax_{x \in \cX} f_i(x)$ is the optimal decision variable for client $i$'s personalized reward function $f_i$. \\

The objective is to design a policy $\pi$ which minimizes the total cumulative regret. We provide high probability regret bounds that hold with probability at least $1 - \delta_0$ for any given $\delta_0 \in (0,1)$, a stronger performance guarantee than bounds on expected regret.

\subsection{Preliminaries on GP Models}\label{sec:GP}

In this section, we overview the GP models and some useful confidence intervals for the RKHS elements based on GP models, which are central to our policy design. \\

A Gaussian Process model, $H(x)$, $x \in \cX$ is a random process indexed on $\cX$, for which all finite subsets $\{H(x_i)\}_{i=1}^n$, $n\in\N$, have a multivariate Gaussian distribution, with mean $\mu(x)=\E[H(x)]$ and covariance $k(x, x')=\E[(H(x)-\mu(x))(H(x')-\mu(x'))]$. We assume $\mu(x)=0$, for all $x\in\cX$. When used as a prior for a data generating process under Gaussian noise, the conjugate property provides closed form expressions for the posterior mean and covariance of the GP model. In particular, given a set of observations $\cH_t=\{\bfx_t, \bfy_t\}$, where $\bfx_t = (x_1, x_2, \dots, x_t)^{\top}$, $\bfy_t = (y_1, y_2, \dots, y_t)^{\top}$, the following expressions can be derived for the posterior mean and covariance of the GP model 
\begin{align}
	\mu_t(x) & = \E \left[H(x) |\cH_t\right] = k_{\bfx_t, x}^{\top} \left( K_{\bfx_t, \bfx_t} + \lambda I_t \right)^{-1} \bfy_t \label{eq:posterior_mean} \\
	k_t(x, x') & = \E\left[(H(x) - \mu_t(x))(H(x') - \mu_t(x'))|\cH_t\right]  \nonumber \\
	& = k(x, x') - k_{\bfx_t, x}^{\top} \left( K_{\bfx_t, \bfx_t} + \lambda I_t \right)^{-1} k_{\bfx_t, x'}. \label{eq:posterior_variance}
\end{align}
In the above expressions, $k_{\bfx_t, x} = [ k(x_1, x), .., k(x_t, x) ]^{\top}$, $K_{\bfx_t, \bfx_t}$ is the $t \times t$ covariance matrix  $[k(x_i, x_j)]_{i,j = 1}^t$, $I_t$ is the $t \times t$ identity matrix and $\lambda$ is the variance of the Gaussian noise. We use $\sigma_t^2(\cdot)=k_t(\cdot,\cdot)$ to denote the posterior variance of the GP model. \\


Following a standard approach in the literature (e.g., see ~\cite{Chowdhury2017, Shekhar2018, Srinivas2010}), we use GPs to model an unknown function $h$ belonging to the RKHS corresponding to the covariance kernel of the GP. In particular, we assume a \emph{fictitious} GP prior $H$ over the \emph{fixed}, unknown function $h$ along with \emph{fictitious} Gaussian distribution for the noise. We would like to emphasize that these assumptions are modelling techniques used as a part of algorithm and not a part of the problem setup. The reason for introducing this fictitious model is that the posterior mean and variance defined above, respectively, provide powerful tools to predict the values of $h$, and to quantify the uncertainty in the prediction. We formalize this statement in the following lemma.

\begin{lemma}[Theorem~$1$ in~\cite{vakili2021optimal}]
Under Assumptions~\ref{assumption1} and~\ref{assumption:sub_gaussian}, provided observations $\cH_t = \{\bfx_t, \bfy_t\}$ as specified above with the query points $\bfx_t$ chosen independently of the noise sequence, we have, for a fixed $x\in\cX$, with probability at least $1-\delta$,
\begin{align*}
|h(x)-\mu_t(x)| \leq \beta(B, \delta)\sigma_t(x),
\end{align*}
where $\beta(B, \delta)=B+R\sqrt{(2/\lambda)\log\left(2/\delta\right)}$.
\label{lemma:concentration_bound}
\end{lemma}
This lemma shows that $\mu_t$ may be used to predict the value of $h$, where the error in the prediction is bounded by a factor of $\sigma_t$ with high probability. In CEPE, $\sigma_t(\cdot)$ is used to guide exploration and $\mu_t(\cdot)$ is used to guide exploitation.   \\

Furthermore, using the predictive variance we also define the maximal information gain, $\gamma_t$, which characterizes the effective dimension of the kernel. It is defined as $\gamma_t := \max_{(x_1, x_2, \dots, x_t) \in \cX^{t}} \sum_{i = 1}^{t} \sigma^2_{i - 1}(x_i)$. Bounds on $\gamma_t$ for several common kernels are known~\citep{Srinivas2012, Vakili2020infogain} and are increasing sublinear functions of $t$, i.e., $\gamma_t = \Oc(t^{\kappa})$ for $\kappa \in (0,1]$.

\section{The CEPE Policy}
\label{section:policy}

In this section, we describe our proposed algorithm for the problem of personalized kernel bandits.

\subsection{Algorithm Description}\label{sec:flex_description}

The design of a collaborative learning policy such as CEPE significantly deviates from those for single client settings. If one were to design a collaborative learning algorithm by simply deploying a classical centralized kernel-based learning algorithm like GP-UCB or GP-TS~\cite{Chowdhury2017} at each client to learn their personalized rewards, then such an algorithm could result in a trivial regret, linearly growing with $T$. This can be attributed to the fact that the personalized objective function $f_i$ of client $i$ contains the local observation functions $\{h_j\}_{j\neq i}$ of other clients that are not observable to client $i$. If client $i$ were to use GP-UCB on their personalized reward, then their query points would quickly concentrate around $x_i^*$ making it difficult for a client $j$ to satisfactorily learn $h_i$ (and hence $f_j$) especially in regions far away from $x_i^*$,  even if all the observations of client $i$ were available to client $j$. Hence, Collaborative Exploration is necessary for each client to learn the maximizer $x_i^*$ of its personalized objective function $f_i$.
This adds another layer of nuance to the exploration-exploitation trade-off as the exploration at any client needs to serve the goal of greater good instead of being sufficient for local exploitation.  \\

Inspired by \emph{deterministic sequencing of exploration and exploitation (DSEE)} algorithm for discrete bandits for centralized learning~\cite{Vakili2013}, we propose CEPE, a kernel-based collaborative learning policy that effectively trades off exploration and exploitation in a distributed learning setting. We first describe the design of the interleaving exploration and exploitation epochs, a central component in CEPE. Following that, we describe the communication protocol and the query policy adopted in CEPE that ensures a low communication cost along with an overall low regret.


\subsubsection{The interleaving epoch structure}

CEPE divides the time horizon of $T$ steps into interleaving epochs of exploration and exploitation of increasing lengths. The length of each epoch is fixed at the beginning of the algorithm based on $N_t:\N\rightarrow \R$, a positive non-decreasing sequence of real numbers that is assumed to have been provided as an input to CEPE. Let $\cA(t)$ denote the collection of time instants at which CEPE had carried out an exploration step up to time $t -1$. CEPE decides to explore at a time instant $t$ if $|\cA(t)| \leq N_t$, otherwise it decides to exploit at time $t$. As a concrete example, consider the case where $N_t = t^{2/3}$ for $t \in N$. At $t = 1$, $|\cA(1)| = 0$ as no point has been explored before it. Since $N_1 = 1 > |\cA(1)|$, CEPE explores at $t = 1$. For $t = 2$, evidently $|\cA(2)| = 1$ and $N_2 = 2^{2/3} > 1$ implying CEPE again explores at $t = 2$. Similarly, CEPE also explores at $t = 3$ as $|\cA(3)| = 2 < N_3$. However, at $t = 4$, $|\cA(4)| = 3$ while $N_4 < 3$. Consequently, CEPE begins an exploitative epoch at $t = 4$ and extends to $t = 5$ as $|\cA(5)| = 3 > N_5$. Once again at $t = 6$, $|\cA(6)| < N_6$ and CEPE decides to explore, terminating the first exploitative epoch and beginning the second exploratory epoch. This process repeats until the end of the time horizon.   \\

The design of the exploration-exploitation epochs is a central piece of the puzzle as it is related to both the communication cost and the regret due to its implicit links with the communication protocol of CEPE and the design of the update rule of the decision rule, as described below.

\subsubsection{Communication Protocol}

CEPE adopts a straightforward, easy to implement communication and message exchange protocol. All the communication between the clients and the server happens only during an exploration epoch. CEPE forgoes communication during exploitation epoch to ensure a low communication overhead. Thus, the separation between exploration and exploitation allows effective control of information exchange across clients for learning the global objective while simultaneously helping limit the communication overhead. \\

At every time instant $t$ during an exploration epoch, each client sends the value of the current decision variable queried, $x_{i ,t}$ along with the observed random reward, $y_{i,t}$ to the server. The server then broadcasts this to all clients. \\


We would like to point out that such a communication scheme does not violate any privacy concerns and is designed in a similar spirit to various distributed learning algorithms, especially in the scenario of Federated Learning. Similar to any FL algorithm, CEPE only communicates the current model parameters (the decision variable $x$ in this setup) and not the actual data. The only difference is that CEPE also communicates a random loss associated with the current decision variable, which also does not give access to any actual data held by the clients. Thus, CEPE also ensures privacy by avoiding data sharing as considered in FL setups.

\subsubsection{Query Policy}

The only step left to specify is the update rule of decision variable, $x$, used in CEPE. The update rule is based on whether the current time instant belongs to an explorative or an exploitative epoch and is guided by a fictitious Gaussian process (GP) prior for the elements of the RKHS, as described in Section~\ref{sec:GP}. During an exploratory epoch, the clients choose different points across the epoch for continual learning. In particular, client $i$ chooses $x_{i,t} = \argmax_{x \in \cX} \sigma_{t-1}^{(h_i)}(x)$, where $\sigma_{t-1}^{(h_i)}(\cdot)$ is the posterior standard deviation of the GP model of $h_i$ based on all the previous exploratory observations $\cS_{i,t} = \{ (x_{i,s}, y_{i,s}) : s \in \cA(t) \}$ of client $i$ according to equation~\eqref{eq:posterior_variance}. On the other hand, during an exploitation epoch the query point is fixed and chosen to maximize earning and forgo learning. Specifically, each client $i$ chooses a point with the highest predicted value based on the GP model, i.e., $x_{i,t} = \argmax_{x\in\cX} \mu_{t-1}^{(f_i)}(x),$ where $\mu_{t-1}^{(f_i)}(\cdot)$ is the posterior mean of the personalized reward function, given as
\begin{align*}
    \mu_{t-1}^{(f_i)}(\cdot) = \alpha_i \mu_{t-1}^{(h_i)}(\cdot) + \frac{(1 - \alpha_i)}{K} \sum_{j = 1}^K \mu_{t-1}^{(h_j)}(\cdot).
\end{align*}
In the above expression $\mu_{t-1}^{(h_j)}(\cdot)$ are the posterior means of the GP model of $h_j$ based on all the previous exploratory observations $\cS_{j,t}$ corresponding to client $j = 1, 2, \dots, K$. We also provide a pseudocode in Alg.~\ref{alg:fed_exp_2} that succinctly combines the three design components of CEPE.

\begin{algorithm}
	\caption{Collaborative Exploration with Personalized Exploitation (CEPE)}
	\label{alg:fed_exp_2}
	\begin{algorithmic}
		\STATE {\bfseries Input:} $\{B_j\}_{j = 1}^{K}$, the kernel $k(\cdot, \cdot)$, $\{N_t\}_{t \in \N}$
		\STATE Set $t \leftarrow 1$, $\cA(1) \leftarrow \emptyset$
	    \REPEAT
	    \IF{$|\cA(t)| < N_t$} 
	    \STATE $\cA(t + 1) = \cA(t) \cup \{t\}$
	    \ENDIF
	    \IF{$t \in \cA(t + 1)$}
	    \STATE $x_{i,t} = \argmax_{x \in \cX} \sigma_{t-1}^{(h_i)}(x)$
	    \ELSE
	    \STATE $x_{i,t} = \argmax_{x\in\cX} \mu_{t-1}^{(f_i)}(x)$
	    \ENDIF
	    \STATE Query the function $h_i$ at $x_{i,t}$ to obtain $y_{i,t}$
	    \IF{$t \in \cA(t+1)$}
	    \STATE Send $(x_{i,t}, y_{i,t})$ to the server which is then broadcast to all clients 
	    \ENDIF
            \STATE $t \leftarrow t + 1$
	    \UNTIL{$t = T$}
	\end{algorithmic}
\end{algorithm}

\subsection{Performance Analysis}

The theorem below establishes an upper bound on the regret performance of CEPE.

\begin{theorem}
Consider the kernel-based collaborative learning with personalized rewards setting described in Section~\ref{sub:fl_personalization}. Under Assumptions~\ref{assumption1} and~\ref{assumption:sub_gaussian} and for a given sequence $\{N_t\}_{t \in \N}$, the regret performance of CEPE satisfies, for any $\delta_0 \in (0,1)$, with probability at least $1-\delta_0$, 
\begin{eqnarray}
R_{\text{CEPE}}(T) =\cO\left(
K N_T + K T\sqrt{\frac{\gamma_{N_T}}{N_T} \log \left( \frac{T}{\delta_0}  \right)} \nonumber
\right)
\end{eqnarray}
In particular, the regret is minimized with the choice of $N_T=\Theta(T^{\frac{2}{3-\kappa}}(\log(T/\delta_0))^{\frac{1}{3}})$ in which case,
\begin{align*}
    R_{\text{CEPE}}(T) =\cO(KT^{\frac{2}{3 - \kappa}}(\log({T/\delta_0}))^{\frac{1}{3}} ). \nonumber
\end{align*}
\label{thm:regret_fed_exp}
\end{theorem}
Recall that $\gamma_t = \Oc(t^{\kappa})$ with $\kappa \in (0,1]$ corresponds to the maximal information gain of the kernel $k$. Specifically, if the underlying kernel is a Matern kernel with smoothness parameter $\nu$, then it is known that $\kappa = d/(2\nu + d)$~\cite{Vakili2020infogain}. Consequently, the regret incurred by CEPE is given as $\cO(T^{\frac{2\nu + d}{3\nu  + d}})$. In the case of a Squared Exponential kernel, $\kappa \to 0$ and hence $R_{\text{CEPE}}(T) = \cO(T^{2/3} (\log (T))^{d/6})$. \\


\begin{proof}
We here provide a sketch of the proof. We bound the regret in the exploration and exploitation epochs separately. The bound on the RKHS norm of $f_i$ implies a bound on its sup norm. Thus, the regret in the exploration epoch is simply bounded by $\Oc(KN_T)$. The regret during the exploitation epoch is bounded using Lemma~\ref{lemma:concentration_bound} based on the posterior variance of the GP model. Maximal uncertainty reduction sampling and a known bound on the total uncertainty (cumulative conditional variances) of a GP model based on information gain allows us to bound the posterior variance with an $\Oc(\frac{\gamma_{N_t}}{N_t})$ term. Combining these two results, we bound the regret in the exploitation epoch. A detailed proof of Theorem~\ref{thm:regret_fed_exp} is provided in Appendix~\ref{proof:regret_fed_exp}.
\end{proof}

\begin{remark}\label{remark:singlekernel}
The upper bound on the regret of CEPE given in Theorem~\ref{thm:regret_fed_exp} is sublinear in $T$ as $\gamma_T=o(T)$ for typical kernels~\cite{Vakili2020infogain}. A sublinear regret bound guarantees the convergence to the maximum personalized reward $f_i(x_i^*)$ across all clients, as $T$ grows. We would like to emphasize the significance of this guarantee for the performance of CEPE. In comparison, it is not clear whether standard algorithms such as GP-UCB or GP-TS achieve sublinear regret, even in the simpler problem of a single client. The existing upper bounds on the regret performance of these algorithms are in the form of $\tilde{\mathcal{O}}(\gamma_T\sqrt{T})$, which may be trivial, as $\gamma_T$ may grow faster that $\sqrt{T}$. For example, that is the case with a broad range of parameters in the case of Mat{\'e}rn kernel (see~\cite{vakili2021open} for a detailed discussion). The CEPE algorithm thus may be of interest even for a single client setting, in terms of introducing a simple algorithm with sublinear regret. We, however, note that more sophisticated algorithms such as GP-ThreDS~\cite{Salgia2020GPThreDS}, SupKernelUCB~\cite{Valko2013a}, and BPE~\cite{li2022gaussian} achieve better regret bounds in the case of a single client.
\end{remark}


We would like to emphasize the role of the sequence $\{N_t\}_{t \in \N}$ in the context of the performance of CEPE, both in terms of regret and communication cost. It affects both these metrics through the overall amount of exploration carried out in CEPE. In particular, at the end of the time horizon, CEPE would have explored for no more than $N_{T}$ time steps. The impact of this on the regret can be noted from the statement of Theorem~\ref{thm:regret_fed_exp}. The first term, $\mathcal{O}(KN_T)$, corresponds to the regret incurred during the exploration sequence, and increases with $N_T$. The second term, $\mathcal{O}\left(KT\sqrt{\frac{\gamma_{N_T}}{N_T}\log(\frac{1}{\delta_0})}\right)$, corresponds to the regret incurred during exploitation, and decreases with $N_T$. That is, a larger exploration sequence enables a better approximation of the arm statistics, which leads to a better performance during the exploitation sequence. On the other hand, since CEPE communicates at all and only the exploration time instants, its communication cost is bounded by $N_{T}$. Thus, CEPE controls the communication-regret trade-off through this sequence $N_t$ and one can appropriately choose the sequence to trade-off the regret against communication based on the demands of a particular application. For example, if the objective is to minimize regret, then the optimal size $N_T$ is obtained by minimizing the larger of the two terms in expression of regret, yielding us $N_T=\Theta(T^{\frac{2}{3- \kappa}}(\log(1/\delta_0))^{\frac{1}{3}})$, as referred to in Theorem~\ref{thm:regret_fed_exp}. Under this scenario, CEPE incurs a communication cost of $\cO(T^{2/(3- \kappa)})$, which is sublinear in $T$.

\section{Reducing Communication Cost via Sparse Approximation}\label{section:sparsepolicy}


In this section, we develop a variant of the CEPE algorithm that leverages sparse approximations of the posterior GP models to achieve significant reduction in communication cost while preserving the optimal regret performance.We begin with some preliminaries on sparse approximations of GP followed by the description of our proposed algorithm.

\subsection{Sparse Approximation of GP models}

The sparse approximations of GP models are designed to approximate the posterior mean and variance obtained from the GP model, using a subset of query points to reduce the computational cost associated with evaluating the exact expressions. Let $\bfz_t = \{z_1, z_2, \dots, z_m\} \subset \bfx_t$ be a subset of the query points, which we aim to use in approximating the posterior distribution of the GP model. These points are often referred to as the \emph{inducing points}. The Nystr{\"o}m approximations are given as follows~\citep{Wild2021}:
\begin{align}
    \tilde{\mu}_t (x) & = k_{\bfz_t, x}^{\top} (\lambda K_{\bfz_t, \bfz_t} +  K_{\bfz_t, \bfx_t} K_{\bfx_t, \bfz_t})^{-1} K_{\bfz_t, \bfx_t} \bfy_t  \label{eqn:approx_posterior_mean}\\
    \tilde{\sigma}^2_t (x) & = \frac{1}{\lambda} \big( k(x, x) - k_{\bfz_t, x}^{\top}  K_{\bfz_t, \bfz_t}^{-1}  k_{\bfz_t, x} +  \nonumber \\
    & \ \ \ \  \   k_{\bfz_t, x}^{\top}( K_{\bfz_t, \bfz_t} +  \lambda^{-1} K_{\bfz_t, \bfx_t} K_{\bfx_t, \bfz_t})^{-1} k_{\bfz_t, x}  
    \label{eqn:approx_posterior_variance}
\end{align}
where $K_{\bfz_t, \bfz_t} = [k(z_i, z_j)]_{i,j = 1}^m \in \R^{m \times m}$, $K_{\bfz_t, \bfx_t} = [k(z_i, x_j)]_{i=1, j= 1}^{m,t} \in \R^{m \times t}$ and $K_{\bfx_t, \bfz_t} = K_{\bfz_t, \bfx_t}^{\top}$.

\subsection{CEPE with Sparse Approximation}
\label{sub:sparse_fed_ex2}

Referred to as S-CEPE, this variant of CEPE 
adopts a simpler design of interleaved exploration-exploitation phases. It consists of a single exploration phase of length $N_T$, entirely concentrated at the beginning of the time horizon, followed by an exploitation phase till the end of the time horizon. Such a design provides the clients with the entire exploration sequence which is required to construct the set of inducing points in Nystr\"om approximation.  \\

Unlike CEPE, the communication in S-CEPE is completely carried out in the exploitation phase. In particular, all the communication is carried during the first $N_T^{(c)}$ steps of the exploitation phase, which are referred to the communication phase. At the $s^{\text{th}}$ instant during this communication phase, client $i$ sends $(z_{i, s}, \bfw_{i, s})$ to the server, where $z_{i,s}$ denotes the $s^{\text{th}}$ inducing point and $\bfw_{i, s}\in \R$ denotes the $s^{\text{th}}$ element of the vector $\bfw_i$. The construction of the inducing points and the vector $\bfw_i$ is described below. The server then broadcasts this to all the clients. The length of the communication sequence, $N_T^{(c)}$, is set to $9(1 + 1/\lambda) q_0 \gamma_{N_T}$ to allow all clients to complete their communication with high probability (See Lemma~\ref{lemma:inducing_set_size}), where $q_0$ is a constant specified later. \\

The messages communicated during S-CEPE are determined using the sparse approximation of GP models. Specifically, at the end of the exploration epoch, each client $i$ constructs its set of inducing points $\bfz_{i, N_T} = (z_{i, 1}, z_{i, 2}, \dots, z_{i, M_i})$ by including each point $x_{i,j} \in \bfx_{i,N_T}$ queried during the exploration epoch with a probability equal to $q_0 \left[\sigma^{(h_i)}_{N_T}(x_{i,j})\right]^2$, independent of other points. Here $q_0$ is a constant that is used to appropriately scale the probabilities of inclusion to ensure a sufficiently large set for a good approximation~\citep{Calandriello2019} and $M_i$ denotes the random cardinality of the set of inducing points of client $i$. Each client $i$, using their set of inducing points, then proceeds to compute the vector $\bfw_i = (\lambda K_{\bfz_{i, N_T}, \bfz_{i, N_T}} +  K_{\bfz_{i, N_T}, \bfx_{i, N_T}} K_{\bfx_{i, N_T}, \bfz_{i, N_T}})^{-1} K_{\bfz_{i, N_T}, \bfx_{i, N_T}} \bfy_{i, N_T}$.  \\

The query policy used for S-CEPE during the exploration phase is same as that for CEPE, that is, to query points based on maximum uncertainty. During the exploitation phase of S-CEPE, each client $i$ chooses to query the maximizer of the predictive mean of their \emph{local} observation function, $\argmax_{x\in\cX} \mu_{N_T}^{(h_i)}(x)$ for the communication phase. After the communication phase, the clients query the maximizer of the Nystr\"om approximate posterior mean of their personalized reward function for the rest of the epoch. The approximate posterior mean of $f_i$ is given as
$$\tilde{\mu}_{N_T}^{(f_i)}(\cdot) = \alpha_i \tilde{\mu}_{N_T}^{(h_i)}(\cdot) + \frac{(1 - \alpha_i)}{K} \sum_{j = 1}^K \tilde{\mu}_{N_T}^{(h_j)}(\cdot)$$ with $\tilde{\mu}_{N_T}^{(h_i)}$'s as defined in~\eqref{eqn:approx_posterior_mean}. 
A pseudocode for S-CEPE is provided in Algorithm~\ref{alg:sparse_fed_exp2}.
 
\subsection{Performance Analysis}

We first establish a bound on the length of the communication sequence. The following lemma, which is a slightly modified version of Theorem~$1$ in~\cite{Calandriello2019}, gives us an upper bound on the largest inducing set among all the clients.

\begin{lemma}\label{lemma:inducing_set_size}
Fix $\delta \in (0,1)$, $\varepsilon \in (0, 1)$ and let $\chi = \frac{1 + \varepsilon}{1- \varepsilon}$. If the set of inducing points is constructed as outlined in Section~\ref{sub:sparse_fed_ex2} with $q_0 \geq 6\chi \log(4TK/\delta)/\varepsilon^2$, then with probability at least $1 - \delta$, $\max_{i} M_i \leq 9(1 + 1/\lambda) q_0 \gamma_{N_T}$.
\end{lemma}
We can conclude from the above lemma that choice of $N_T^{(c)}$ used in S-CEPE allows for sufficient time for all the clients to transmit their messages to the server in the communication phase. As a result, the communication cost of S-CEPE can be bounded by $N_T^{(c)} = \cO(\gamma_{N_T}) = \cO(N_T^{\kappa})$. The communication cost of S-CEPE is significantly lower than that of CEPE. For example in the case of SE kernel with the optimal choice of $N_T = \cO(KT^{\frac{2}{3}}(\log({T/\delta_0}))^{(d+2)/6})$, the communication cost of S-CEPE is bounded by $\Oc(K\log^{d+1}(T))$, while the communication cost of CEPE is in $\cO(KT^{\frac{2}{3}}(\log({T/\delta_0}))^{(d+2)/6})$. Furthermore, S-CEPE achieves this significant reduction in communication cost while maintaining the regret guarantees of CEPE, as shown in the following theorem.

\begin{theorem}
Consider the kernel-based collaborative learning with personalized rewards as described in Section~\ref{sub:fl_personalization}.
Under Assumptions~\ref{assumption1} and~\ref{assumption:sub_gaussian}, the regret performance of S-CEPE satisfies
\begin{align*}
R_{\text{S-CEPE}}(T) = \cO\left( K N_T + K T\sqrt{\frac{\gamma_{N_T}}{N_T} \log \left( T/\delta_0  \right)} \right).
\end{align*}
with probability at least $1-\delta_0$ for any $\delta_0 \in (0,1)$. 
\label{thm:sparse_fed_exp2}
\end{theorem}

\begin{proof}
The proof of this theorem is similar to that of Theorem~\ref{thm:regret_fed_exp} and we provide a sketch of the proof below. We bound the regret in exploration, communication and the exploitation phases separately. The regret during the exploration phase and the communication phase is bounded within a constant factor of their length, i.e., $\cO(KN_T)$ and $\cO(K\gamma_{N_T})$ respectively.
The exploitation regret is bounded following similar steps as in the proof of~\ref{thm:regret_fed_exp}, except that we bound the instantaneous regret based on approximate posterior standard deviation in contrast to the exact one using Lemma~\ref{lemma:concentration_bound_sparse}, described below. 
\end{proof}

The following lemma is a counterpart to Lemma~\ref{lemma:concentration_bound}, in the case of sparse approximation of the posterior mean and variance.


\begin{algorithm}[!h]
	\caption{S-CEPE}
	\label{alg:sparse_fed_exp2}
	\begin{algorithmic}
		\STATE {\bfseries Input:} $N_T$, the kernel $k(\cdot, \cdot)$, $\varepsilon \in (0, 1)$, $\delta_0 \in (0,1)$, $\lambda$.
		\STATE $\bfz_{i, N_T} \leftarrow \{ \}$, $q_0 \leftarrow 6(1 + \varepsilon) \log(8TK/\delta_0)/\varepsilon^2(1 - \varepsilon)$,
		$N_T^{(c)} \leftarrow 9(1 + 1/\lambda) q_0 \gamma_{N_T}$
		\FOR{$t = 1,2, \dots, N_T$}
		\STATE Choose $x_{i,t} = \argmax_{x \in \cX} \sigma_{t-1}^{(h_i)}(x)$
		\STATE Query the function $h_i$ at $x_{i,t}$ to obtain $y_{i,t}$
		\ENDFOR
		\FOR{$j = 1,2, \dots, N_T$}
		\STATE Draw $W_j \sim \text{Bern}(p_j)$, where $p_j := q_0 [\sigma^{(h_i)}_{N_T}(x_{i, j})]^2$
		\IF{$W_j = 1$}
		\STATE $\bfz_{i, N_T} \leftarrow \bfz_{i, N_T} \cup \{ x_{i, j} \}$
		\ENDIF
		\ENDFOR
		\STATE Evaluate the vector $\bfw_i = (\lambda K_{\bfz_{i, N_T}, \bfz_{i, N_T}} +  K_{\bfz_{i, N_T}, \bfx_{i, N_T}} K_{\bfx_{i, N_T}, \bfz_{i, N_T}})^{-1} K_{\bfz_{i, N_T}, \bfx_{i, N_T}} \bfy_{i, N_T}$
		\FOR{$s = 1, 2, \dots, N_T^{(c)}$}
		\STATE $t \leftarrow t + 1$
		\STATE $x_{i, t} \leftarrow \argmax_{x \in \cX} {\mu}_{N_T}^{(h_i)}(x)$
		\IF{$s \leq m_i$}
		\STATE Send $(z_{i, s} , \bfw_{i, s})$ to the server which is then broadcast to all clients
		\ENDIF
		\ENDFOR
		\REPEAT 
		\STATE $t \leftarrow t + 1$
		\STATE  $x_{i, t} \leftarrow \argmax_{x \in \cX} \tilde{\mu}_{N_T}^{(f_i)}(x)$ 
		\UNTIL{$t = T$}
	\end{algorithmic}
\end{algorithm}

\begin{lemma}
Under Assumptions~\ref{assumption1} and~\ref{assumption:sub_gaussian}, provided a set of observations $\cH_t = \{\bfx_t, \bfy_t\}$, with $\bfx_t$ chosen independently of the associated noise sequence, and corresponding subset of inducing points $\bfz_t$ obtained as outlined in Section~\ref{sub:sparse_fed_ex2},  we have, for a fixed $x\in\cX$, with probability at least $1-\delta$,
\begin{align*}
    |h(x)-\tilde{\mu}_t(x)| \leq \tilde{\beta}(B, \delta) \tilde{\sigma}_t(x),
\end{align*}
where $\tilde{\mu}(x)$ and $\tilde{\sigma}(x)$ are as defined in~\eqref{eqn:approx_posterior_variance} and $\tilde{\beta}(B, \delta)=B\sqrt{2\lambda/(1 - \varepsilon) } + R \sqrt{2 \log ( T/\delta )}$.
\label{lemma:concentration_bound_sparse}
\end{lemma}
Detailed proofs of Theorem~\ref{thm:sparse_fed_exp2} and Lemma~\ref{lemma:concentration_bound_sparse} are provided in Appendix~\ref{proof:sparse_fed_exp2}.

\section{Regret Lower Bound}\label{section:lowerbound}

In this section, we establish a lower bound on regret incurred by any algorithm for the problem of kernel-based collaborative learning with personalized rewards. In particular, in the following theorem, we show that even when there is no communication constraint, the total regret incurred by any learning algorithm as defined in~\eqref{eqn:regret_definition} is at least $\Omega(T^{2/(3 - \kappa)})$, where $\kappa$ characterizes the rate of maximal information gain and depends on the underlying kernel.


\begin{theorem}\label{Theorem:LB}
Consider the kernel-based collaborative learning with personalized rewards setting described in Section~\ref{sub:fl_personalization}, 
in an RKHS corresponding to the family of Squared Exponential (SE) and Matern (with smoothness $\nu$) kernels with no constraints on communication. Then, for all sets of personalization parameters $\{\alpha_1, \alpha_2, \dots, \alpha_K\}$, there exists a choice of observation functions $\{h_1, h_2, \dots, h_K\}$ for which, the cumulative regret of any policy $\pi$, satisfies
\begin{align*}
    \E[R_{\pi}(T)] = \begin{cases}
    \Omega(\alpha_{*} T^{\frac{2\nu + d}{3\nu + d}}) & \text{ for Matern kernels}, \\ \Omega(\alpha_{*} T^{2/3} (\log T)^{d/6}) & \text{ for SE kernel}.
    \end{cases}
\end{align*}
for sufficiently large time horizon $T$. In the above expression, $\alpha_{*} = \max_{i}(\min\{\alpha_i, 1 - \alpha_i\}) > 0$ is a constant that depends only on the personalization parameters. 
\end{theorem}

Theorem~\ref{Theorem:LB} shows that the regret performance of CEPE (and S-CEPE) is order optimal up to logarithmic factors. Please refer to Appendix~\ref{proof:LB} for a detailed proof. 

\begin{remark}
\cite{Shi2021Personalization} also conjectured a lower bound for the regret in the finite arm setting. Their lower bound is, however, different in the sense that it is a distribution-dependent lower bound. The distribution-dependent lower bounds typically are given in terms of the gaps (in the mean and KL-divergence) between the best arms and suboptimal arms. These bounds do not apply to our continuum arm setting as such gaps are always zero in case of infinite arms. We thus use a different method to prove a minimax bound on the regret. 
\end{remark}

\section{Empirical Studies}

In this section, we corroborate our theoretical results with empirical evidence that compares the regret incurred by CEPE against several baseline algorithms and documents the reduction in communication cost offered by S-CEPE. We begin with describing the experimental setup followed by discussing the particular experiments and their results.

\subsection{Experimental Setup}
We consider a collaborative learning setup among $K = 50$ clients. The personalization parameter of each client is drawn uniformly from the interval $[0.1, 0.9]$, independent of all other clients. The observation function for each client is drawn uniformly at random from the following set of $9$ functions: $\{h: h(x_1, x_2) = \Lambda(x_1^{i}, x_2^{j}), i,j \in \{1,2,3\}\}$, where $\Lambda: [0,1]^2\to \R$ is a standard two-dimensional benchmark function for Bayesian Optimization. We vary the choice of $\Lambda$ across experiments and specify it along with the corresponding experiment. We implement all algorithms using a uniform discretization of $900$ points over the domain, $[0, 1]^2$. We use a SE kernel with lengthscale parameter $0.2$. The observations are corrupted with Gaussian noise with $\sigma^2 = 0.01$, and $\lambda$ is set to $0.01$. The results in all the experiments are obtained for a time horizon of $T = 2000$ steps by averaging over $5$ Monte Carlo runs.

\begin{figure*}[h]
\centering
\subfloat[Comparison of regret and communication cost in S-CEPE]{\label{fig:sparseflex_plot}\centering \includegraphics[scale = 0.335]{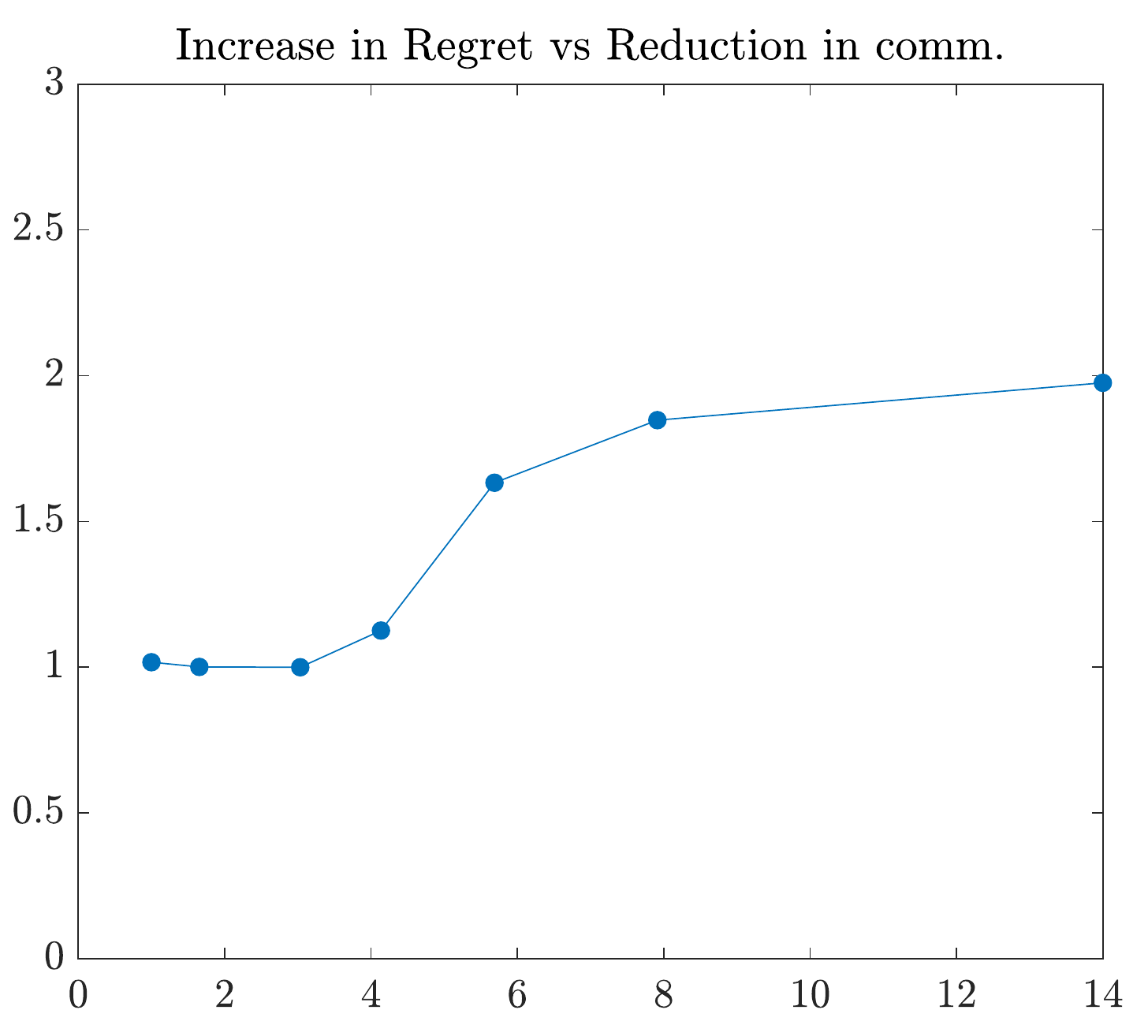}}
~
\subfloat[Regret incurred by different algorithms]{\label{fig:comparison_plot_branin}\centering \includegraphics[scale = 0.265]{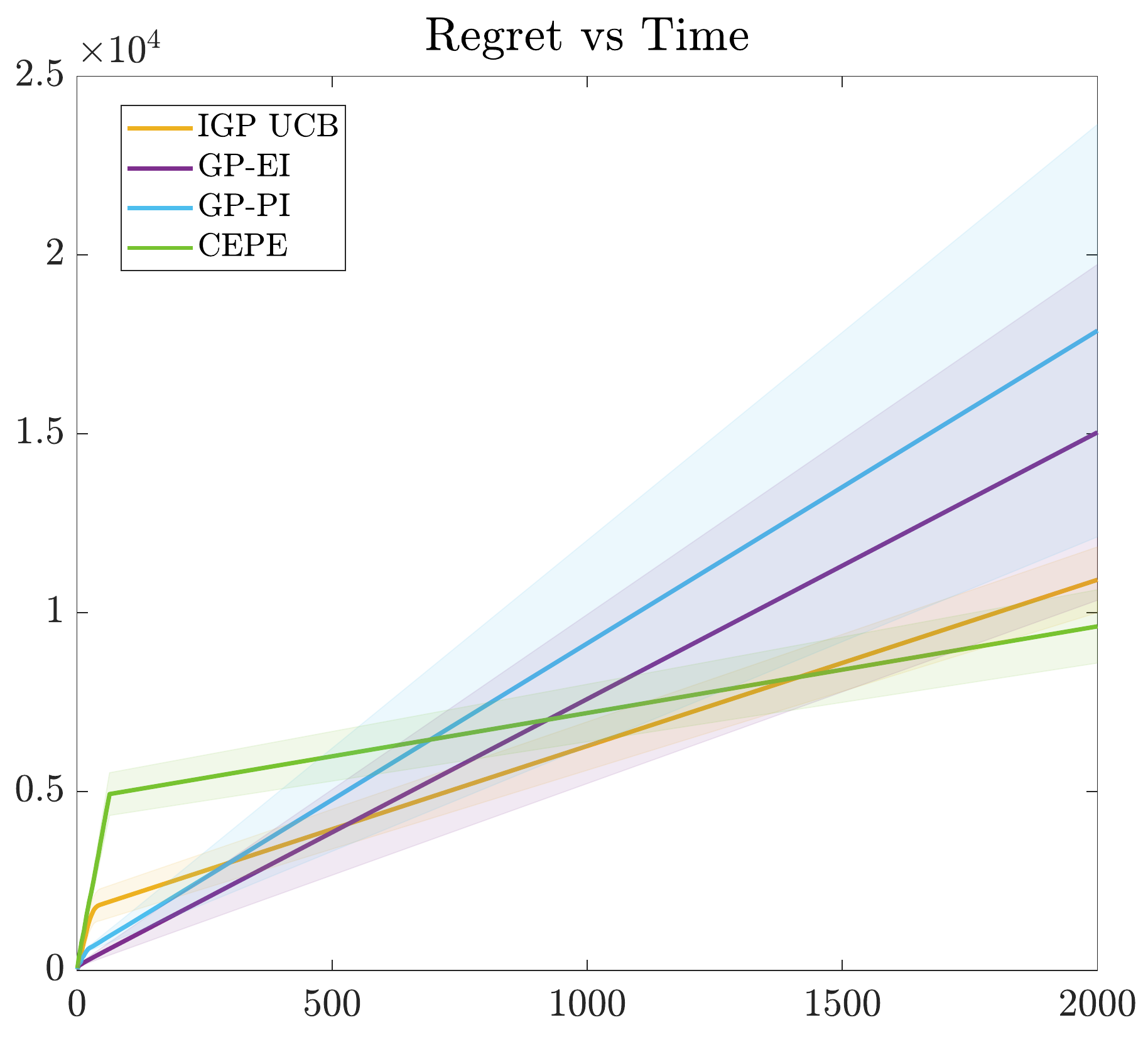}}
~
\subfloat[Regret incurred by different algorithms]{\label{fig:comparison_plot_camel}\centering \includegraphics[scale = 0.3]{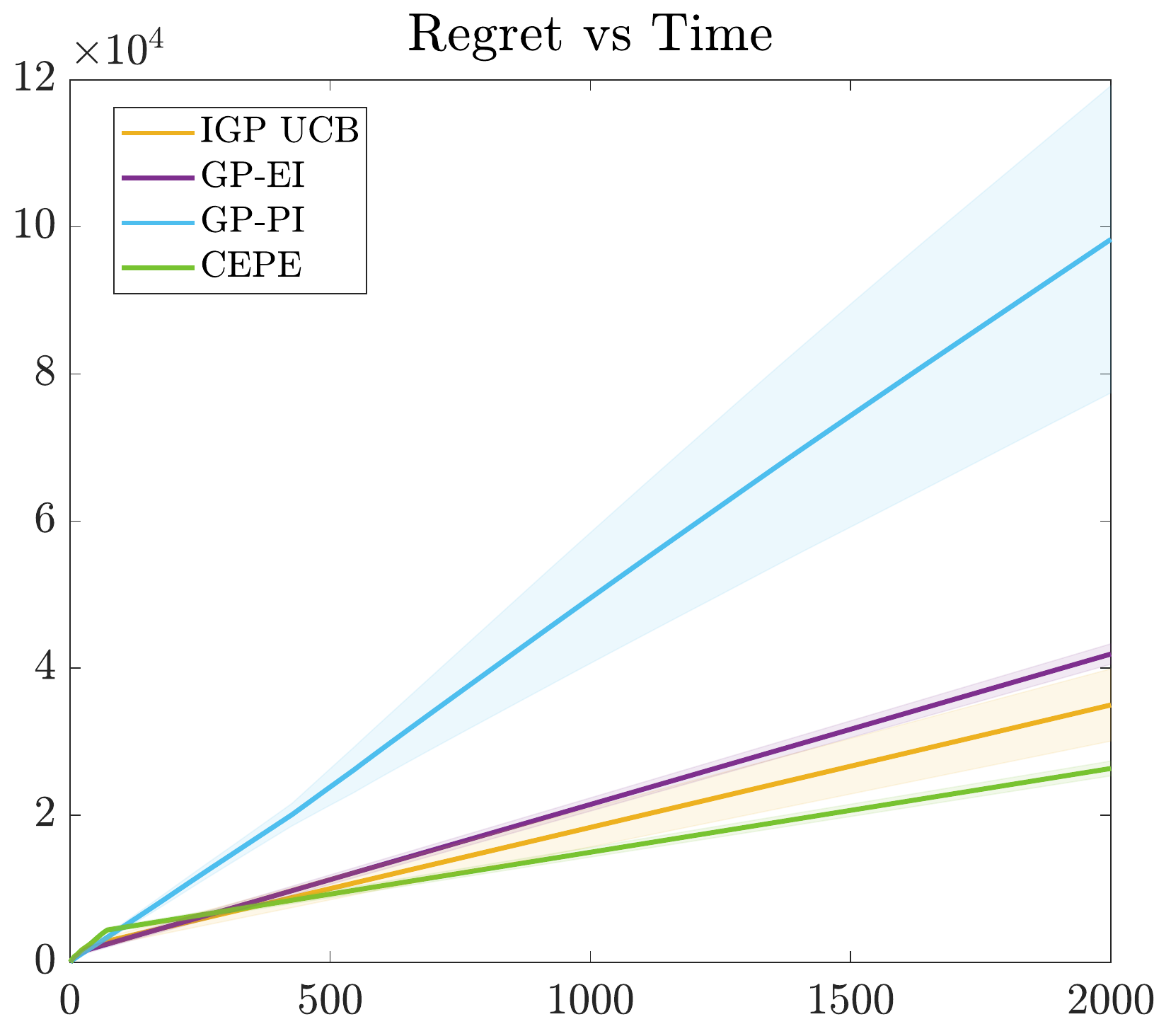}}
~
\caption{(a) Cumulative regret of S-CEPE after $T = 2000$ steps relative to CEPE plotted against the reduction in communication offered relative to CEPE when S-CEPE is run with different sizes of inducing sets. (b, c) Cumulative regret incurred by different algorithms over $T = 2000$ steps for different objective functions ((b) Branin, (c) Function in~\cite{sobester2008engineering})}
\label{fig:plots}
\end{figure*}

\subsection{Results}
In the first experiment, we compare the regret performance of CEPE against baseline algorithms: IGP-UCB~\citep{Chowdhury2017}, GP-EI (Expected Improvement)~\citep{wang2014theoreticalGPEI} and GP-PI (Probability of Improvement)~\citep{wang2018regretGPPI}. Since these baseline algorithms are designed for the single client setting, we appropriately modify them to adapt them to the collaborative learning setup considered in this work. Specifically, for the baseline algorithms we consider the benign setting where at each time instant $t$, the entire history of all the query points and observations of each client upto time $t - 1$ is accessible by all other clients. Each client $i$ uses this history of observations upto time $t - 1$ to construct the estimate of the observation functions of other clients which in turn is used to construct the acquisition function for the query point at time instant $t$ to maximize its personalized objective function, $f_i$. The major difference here is that query points are chosen \emph{greedily} by the clients to maximize their own objective function which might make it difficult for other clients to satisfactorily estimate (and optimize) their own objective functions. Please refer to Appendix~\ref{sec:experiments} for a more detailed description of parameter settings for the different algorithms. \\

We perform this comparison for two different choices of observation functions. In the first case, the base function $\Lambda$ is set to the Branin function~\citep{Azimi2012, Picheny2013} and the overall cumulative regret incurred by different algorithms is plotted in Fig.~\ref{fig:comparison_plot_branin}. In the second case, $\Lambda$ is set to the two-dimensional extension of the benchmark function proposed in~\cite{sobester2008engineering} and overall the cumulative regret for different algorithms for this experiment is plotted in Fig.~\ref{fig:comparison_plot_camel}. As it can be seen from both the plots, the regret incurred by CEPE is lesser than that incurred by other baseline algorithms. The results are consistent with our discussion in Section~\ref{sec:flex_description} showing that baseline GP bandit algorithms do not perform well in the collaborative setting, even with communication at each round. The altruistic observations obtained during the exploration sequence in CEPE helps all the clients and contribute towards an overall lower cumulative regret. \\

In the second experiment, we study the trade-off between communication cost and regret offered by S-CEPE. In particular, we evaluate the overall cumulative regret of S-CEPE for varying sizes of the set of inducing points obtained by varying $q_0$. Recall that the number of inducing points determines the communication cost. This is compared against the regret and communication cost incurred by CEPE. For both CEPE and S-CEPE, we set $N_T = 64$ which also determines the communication cost of CEPE. The objective functions for this case are obtained by setting $\Lambda$ to the Branin function. We plot the ratio of overall cumulative regret between S-CEPE and CEPE against the ratio between the communication cost of CEPE to S-CEPE in Fig.~\ref{fig:sparseflex_plot}. As seen from the plot, S-CEPE offers upto a $14$-fold reduction in communication cost while sacrificing regret by a factor of less than $2$. This shows the efficiency of sparse approximation methods in practice.

\section{Conclusion}

In this work, we formulated the personalization framework for kernelized bandits that allows to tackle the problem of statistical heterogeneity in distributed learning settings. Within this framework, each client can choose to trade-off between the generalization capabilities of a global reward function and a locally focused reward function that aims to fit local data only. For this framework, we proposed a new algorithm called CEPE that achieves a regret of $\tilde{\cO}(T^{2/(3- \kappa)})$, where $\kappa$ depends on the maximal information gain corresponding to the underlying kernel. CEPE builds upon fundamental insight that in such a collaborative learning setup with personalized objective functions, each client needs to take altruistic actions in order to help everyone optimize their own objective function. We also provide a matching lower bound, establishing the order optimality of the regret incurred by CEPE. We also propose a variant of CEPE, called S-CEPE that significantly reduces the communication cost and is based on sparse approximation of GP models. An interesting follow-up direction to analyze the behaviour of neural nets in distributed learning by connecting it to the insight from this result using the concept Neural Tangent Kernel~\cite{jacot2018neural}. 


\bibliography{citations}
\bibliographystyle{icml2022}

\newpage
\appendix

\onecolumn

\section{Proof of Theorem~\ref{thm:regret_fed_exp}}
\label{proof:regret_fed_exp}

To bound the regret of CEPE, we consider the regret incurred during the exploration and exploitation epochs separately. The regret corresponding to the exploration epoch is bounded with a constant factor of the cardinality of the exploration epoch, yielding the $K N_T$ term. The bound on the regret corresponding to the exploitation epoch in obtained in two steps. In the first step, we use the choice of $x_{i, t}$ and Lemma~\ref{lemma:concentration_bound} to bound $|f_i(x_i^*) - f(x_{i,t})|$ in terms of the posterior standard deviation at $x_i^*$ and $x_{i,t}$.  It is notable that, in this analysis, since $x_{i, t}$ in the exploration epoch of CEPE are selected in a purely exploratory way, we can use Lemma~\ref{lemma:concentration_bound} which provides tighter confidence intervals compared to the ones used in the analysis of GP-UCB and GP-TS~\cite{Chowdhury2017}. That is the key in proving an always sublinear regret bound for CEPE. See also~\cite{vakili2021optimal} for further discussions on the confidence intervals. In the second step, we show that the point selection rule in the exploration epoch, based on maximally reducing the uncertainty, allows us to bound both of the posterior standard deviation terms mentioned above by $\sqrt{\gamma_{N_t}/N_t}$, up to an absolute constant factor. We arrive at the theorem by summing the instantaneous regret over $i$ and $t$. \\

Before bounding the regret, we first obtain an upper bound on the RKHS norm of the objective functions of the clients. For any client $i \in \{1,2,\dots, K\}$, we have,
\begin{align*}
    \|f_i\|_{H_k} & = \left\| \alpha_i h_i + \frac{1 - \alpha_i}{K} \sum_{j = 1}^K h_j\right\|_{H_k} \\
    & \leq \alpha_i \|h_i\|_{H_k} + \frac{1 - \alpha_i}{K} \sum_{j = 1}^K \|h_j\|_{H_k} \\
    & \leq \alpha_i B_i + \frac{1 - \alpha_i}{K} \sum_{j = 1}^K B_j.
\end{align*}
For each $i \in \{1,2,\dots, K\}$ define $\displaystyle \bar{B}_i := \alpha_i B_i + \frac{1 - \alpha_i}{K} \sum_{j = 1}^K B_j$. Thus, $\bar{B}_i$ is an upper bound on the RKHS norm of the objective functions. \\

We first consider the regret incurred during the exploration epoch. Let $R_1$ denote the regret incurred during the exploration epoch. We then have
\begin{align*}
    R_1 & = \sum_{i = 1}^K \sum_{t \in \cA(T)} f_i(x_i^*) - f_i(x_{i,t}) \\
     & \leq \sum_{i = 1}^K \sum_{t \in \cA(T)} 2\bar{B}_i \\
     & \leq \sum_{i = 1}^K 2\bar{B}_i |\cA(T)| \\
     & \leq 2 (\max_i \bar{B}_i) K N_T.
\end{align*}
In the second step, we have used $\sup_{x \in \cX} f(x) \leq \|f\|_{H_k}$, which can be shown as follows
\begin{align*}
\sup_{x \in \cX} f(x) & =  \sup_{x \in \cX} \ip{f, k(\cdot, x)}\\
&\leq \sup_{x \in \cX} \|f \|_{H_k} \|k(\cdot, x)\|_{H_k}\\
&\leq \left(\sup_{x\in\cX}k(x,x) \right)\|f\|_{H_k}.
\end{align*}
Using $k(x,x) \leq 1$ for all $x \in \cX$ gives us the required bound. \\


We now consider the regret incurred during the exploitation epoch, which we denote by $R_2$. We first consider the regret incurred by a single client $i$. Based on Assumption~\ref{assumption:discretization}, we consider a discretization $\cD_T$ with $|\cD_T| = \cO(T^d)$ such that for all $x \in \cX$, $i \in \{1, 2, \dots, K\}$ and $t \leq T$, we have, $|h_i(x) - h_i([x]_{\cD_T})| \leq N_T^{-1/2}$ and $|\mu_{t}^{(h_i)}(x) - \mu_{t}^{(h_i)}([x]_{\cD_T})| \leq N_T^{-1/2}$, where $[x]_{\cD_T}$ denotes the point closest to $x$ on the discretization $\cD_T$. Note that we do not need to explicitly construct such a discretization for the actual algorithm. It is considered only for the purpose of analysis. \\

Using Lemma~\ref{lemma:concentration_bound} and a union bound over all the points in the discretization $\cD_T$, we can conclude that with probability at least $1 - \delta_0/K$, $|h_i(x') - \mu_{t-1}^{(h_i)}(x')| \leq \beta(B_i, \delta_0') \sigma_{t-1}^{(h_i)}(x')$ holds for all $x' \in \cD_T$, $t \leq T$ and given client $i \in \{1, 2, \dots, K\}$ where $\delta_0' = \delta_0/(2KT|\cD_T|)$. Consequently, for any client $i$ we have,
\begin{align*}
    |f_i(x') - \mu_{t-1}^{(f_i)}(x')| & = \left| \left(\alpha_i h_i(x') + \frac{1 - \alpha_i}{K} \sum_{j = 1}^K  h_j(x') \right) - \left(\alpha_i \mu_{t-1}^{(h_i)}(x') + \frac{1 - \alpha_i}{K} \sum_{j = 1}^K \mu_{t-1}^{(h_j)}(x') \right)  \right| \\
    & \leq \alpha_i |h_i(x') - \mu_{t-1}^{(h_i)}(x')|  + \frac{1 - \alpha_i}{K} \sum_{j = 1}^K |h_j(x') - \mu_{t-1}^{(h_j)}(x')| \\
    & \leq \alpha_i \cdot  \beta(B_i, \delta_0') \cdot \sigma_{t}^{(h_i)}(x') + \frac{1 - \alpha_i}{K} \sum_{j = 1}^K \beta(B_j, \delta_0')  \cdot\sigma_{t}^{(h_j)}(x').
\end{align*}
The above relation holds for all $i \in \{1,2,\dots, K\}$, $x' \in \cD_T$, $t \leq T$ with probability at least $1 - \delta_0$, which follows immediately from the union bound. With a slight abuse of notation, we define $\displaystyle \sigma_{t-1}^{(f_i)}(x') := \alpha_i \cdot  \beta(B_i, \delta_0') \cdot \sigma_{t-1}^{(h_i)}(x') + \frac{1 - \alpha_i}{K} \sum_{j = 1}^K \beta(B_j, \delta_0')  \cdot\sigma_{t-1}^{(h_j)}(x')$. \\

With this relation, we move onto bound the regret incurred by any client $i$ during the exploitation epoch. We have,
\begin{align*}
    \sum_{t \notin \cA(T)} f_i(x_i^*) - f_i(x_{i,t}) & = \sum_{t \notin \cA(T)} f_i(x_i^*) - f_i([x_{i,t}]_{\cD_T}) + f_i([x_{i,t}]_{\cD_T}) - f_i(x_{i,t}) \\
    & \leq \sum_{t \notin \cA(T)} f_i(x_i^*) - f_i([x_{i,t}]_{\cD_T}) + N_T^{-1/2} \\
    & \leq \sum_{t \notin \cA(T)} f_i(x_i^*) - \mu_{t-1}^{(f_i)}(x_i^*) + \mu_{t-1}^{(f_i)}(x_{i,t})  - f_i([x_{i,t}]_{\cD_T}) + N_T^{-1/2} \\
    & \leq \sum_{t \notin \cA(T)} f_i(x_i^*) - \mu_{t-1}^{(f_i)}(x_i^*) + \mu_{t-1}^{(f_i)}(x_{i,t}) -  \mu_{t-1}^{(f_i)}([x_{i,t}]_{\cD_T}) + \mu_{t-1}^{(f_i)}([x_{i,t}]_{\cD_T})  - f_i([x_{i,t}]_{\cD_T}) + N_T^{-1/2} \\
    & \leq \sum_{t \notin \cA(T)} f_i(x_i^*) - \mu_{t-1}^{(f_i)}(x_i^*) + \mu_{t-1}^{(f_i)}([x_{i,t}]_{\cD_T})  - f_i([x_{i,t}]_{\cD_T}) + 2N_T^{-1/2} \\
    & \leq \sum_{t \notin \cA(T)} \sigma_{t-1}^{(f_i)}(x_i^*) + \sigma_{t-1}^{(f_i)}([x_{i,t}]_{\cD_T}) + 2N_T^{-1/2}.
\end{align*}
In the third step, we used the fact that $x_{i,t} = \argmax_{x \in \cX} \mu_{t-1}^{(f_i)}(x) $. The expression on the RHS in the above equation is bounded using the following lemma. \\

\begin{lemma}
Consider the epoch of points $\{x_1, x_2, \dots, \}$ generated using the maximum posterior variance sampling scheme, that is, $x_t = \argmax_{x \in \cX} \sigma_{t - 1}(x)$, where $\sigma_{t-1}^2$ is the posterior variance computed from the history $\{x_1, x_2, \dots, x_{t-1}\}$. Then, at any time $t$, the following inequality is true.
\begin{align*}
    \sup_{x \in \cX} \sigma_{t}(x) \leq \sqrt{\frac{12\gamma_t}{t}},
\end{align*}
where $\gamma_t$ is the information gain after $t$ steps.
\label{lemma:posterior_std_dev_upper_bound}
\end{lemma}

The lemma provides an upper bound on the posterior standard deviation under the maximum posterior variance sampling scheme, as adopted in the exploration epoch of CEPE. The proof of this lemma is provided at the end of this section. \\

Note that $t \notin \cA(T) \implies t \notin \cA(t)$ which in turn implies that $|\cA(t)| > N_t$. Thus, at any time instant $t$ in the exploitation epoch, the posterior standard deviation is computed based on a history of at least $N_t$ observations taken according to the maximum posterior variance sampling scheme. Hence, using Lemma~\ref{lemma:posterior_std_dev_upper_bound} we can conclude that 
$\sigma_{t-1}^{(h_i)}(x) \leq \sqrt{12 \gamma_{N_t}/N_t}$ for all $x \in \cX$ and all clients $i \in \{1,2,\dots, K\}$. Consequently, for any client $i$ and any $x \in \cX$, 
\begin{align*}
    \sigma_{t-1}^{(f_i)}(x) & = \alpha_i \cdot  \beta(B_i, \delta_0') \cdot \sigma_{t-1}^{(h_i)}(x) + \frac{1 - \alpha_i}{K} \sum_{j = 1}^K \beta(B_j, \delta_0')  \cdot\sigma_{t-1}^{(h_j)}(x) \\
    & \leq \alpha_i \cdot  \beta(B_i, \delta_0') \cdot \sqrt{\frac{12 \gamma_{N_t}}{N_t}} + \frac{1 - \alpha_i}{K} \sum_{j = 1}^K \beta(B_j, \delta_0')  \cdot \sqrt{\frac{12 \gamma_{N_t}}{N_t}}\\
    & \leq \sqrt{\frac{12 \gamma_{N_t}}{N_t}} \cdot \left( \alpha_i \cdot  \beta(B_i, \delta_0') + \frac{1 - \alpha_i}{K} \sum_{j = 1}^K \beta(B_j, \delta_0') \right) \\
    & \leq \sqrt{\frac{12 \gamma_{N_t}}{N_t}} \cdot \left( \alpha_i \cdot  \beta(B_i, \delta_0') + \frac{1 - \alpha_i}{K} \sum_{j = 1}^K \beta(B_j, \delta_0') \right) \\
    & = \sqrt{\frac{12 \gamma_{N_t}}{N_t}} \cdot \beta(\bar{B}_i, \delta_0'),
\end{align*}
where the last step follows from the definition of $\beta(B, \delta)$ and $\bar{B}_i$. On plugging this bound in the expression for the regret incurred by a single client during the exploitation epoch, we obtain the following bound on $R_2$, the overall regret incurred during the exploitation epoch, which holds with probability of at least $1 - \delta_0$.
\begin{align*}
    R_2 & = \sum_{i = 1}^K \sum_{t \notin \cA(T)} f_i(x_i^*) - f_i(x_{i,t}) \\
    & \leq \sum_{i = 1}^K \sum_{t \notin \cA(T)} \sigma_{t-1}^{(f_i)}(x_i^*) + \sigma_{t-1}^{(f_i)}([x_{i,t}]_{\cD_T}) + 2N_T^{-1/2} \\
    & \leq \sum_{i = 1}^K \sum_{t \notin \cA(T)} \left(\sqrt{\frac{12 \gamma_{N_t}}{N_t}} \cdot \beta(\bar{B}_i, \delta_0') + \sqrt{\frac{12 \gamma_{N_t}}{N_t}} \cdot \beta(\bar{B}_i, \delta_0') + 2N_T^{-1/2} \right) \\
    & \leq 2 K  T \left[ \left( \max_{i} \bar{B}_i + R \sqrt{\frac{2}{\lambda} \log \left( \frac{2}{\delta_0'}\right)}\right)  \sqrt{\frac{12\gamma_{N_T}}{N_T}} + \frac{1}{\sqrt{N_T}}\right],
\end{align*}
The last step follows by noting that 
\begin{align*}
    \sum_{t \notin \cA(T)} \sqrt{\frac{\gamma_{N_t}}{N_t}} & \leq T \cdot  \left( \frac{1}{T} \sum_{t = 1}^T \sqrt{\frac{\gamma_{N_t}} {N_t}}\right) \leq T \cdot   \sqrt{\frac{\gamma_{N_T}} {N_T}},
\end{align*}
where the last inequality is a result of applying Jensen's inequality on the concave function $t \to \sqrt{\gamma_t/t}$ and the monotonicity of $N_t$. On adding the bounds on $R_1$ and $R_2$, we arrive at the theorem. 



\subsection{Proof of Lemma~\ref{lemma:posterior_std_dev_upper_bound}}

The proof is based on Lemma 4 of~\cite{Chowdhury2017} which bounds the sum of sequential posterior standard deviations of a GP model as follows
\begin{align*}
    \sum_{s = 1}^t \sigma_{s-1}(x_s) \leq \sqrt{12 t \gamma_t}.
\end{align*}
This is true for all $t \geq 1$ and for any sampling scheme that sequentially generates the points $\{x_1, x_2, \dots\}$. In the case of maximally reducing the variance, for any choices of $s' \leq s$, we have 
\begin{eqnarray*}
    \sigma_{s'-1}(x_{s'}) &\geq&\sigma_{s'-1}(x_{s})\\ &\geq& \sigma_{s-1}(x_{s}).
\end{eqnarray*}
The first inequality comes from the selection rule. The second inequality comes from the fact that conditioning a GP on a larger set of points reduces the variance (that follows from the positive definiteness of the covariance matrix).  \\

This implies that $\sigma_{t-1}(x_t)$ is the smallest term in the sum $\sum_{s = 1}^t \sigma_{s-1}(x_s)$ and hence has to be smaller than the mean. Consequently, we have,
\begin{align*}
    \sup_{x \in \cX} \sigma_{t-1}(x) = \sigma_{t-1}(x_t) \leq \frac{1}{t} \sum_{s = 1}^t \sigma_{s-1}(x_s) \leq \sqrt{\frac{12 \gamma_t}{t}},
\end{align*}
as required.

\section{Proof of Theorem~\ref{thm:sparse_fed_exp2}}
\label{proof:sparse_fed_exp2}

We first introduce some additional notation corresponding the feature representation of functions in a RKHS, which is used in a part of the proof. In particular, the kernel function $k(\cdot, \cdot)$ is associated with a nonlinear feature map $\phi(\cdot): \R^d \to \cH_k$ such that $k(x, x') = \phi(x')^{\top} \phi(x)$. Consequently, the reproducing property can also be written as $h(x) = \phi(x)^{\top} h$. Corresponding to a client $i$, at each time instant $t$, we define the matrix $\Phi_{i,t} = [\phi(x_{i,1}), \phi(x_{i,2}), \dots, \phi(x_{i,t})]^{\top}$ based on the points sampled by client $i$, i.e., $\bfx_{i,t}$.  \\

Let $\bfz_{i,t} = \{z_{i,1}, z_{i,2}, \dots, z_{i, m_i}\}$ denote the set of inducing points chosen from $\bfx_{i,t}$ according to the strategy outlined in Section~\ref{section:sparsepolicy}. Let $S_{i, t}$ be the diagonal $\R^{t \times t}$ matrix with $W_{i,j}/\sqrt{p_{i,j}}$ as the $j^{\text{th}}$ element on the diagonal. Recall that $p_{i,j} = q_0 \left[ \sigma_{t-1}^{(h_i)}(x_{i,j}) \right]^2$ was the probability of choosing point $x_{i, j}$ for $j = 1,2,\dots, t$ and $W_{i,j}$ was the corresponding $\{0,1\}$ random variable that determined whether that point was a part of the inducing set $\bfz_{i,t}$ or not. Consequently, we have the following relation - 
\begin{align*}
    \sum_{z \in \bfz_{i,t}} \frac{1}{p_{i,j}} \phi(z) \phi(z)^{\top} = \sum_{j =1}^{t}  \frac{W_{i,j}}{p_{i,j}} \phi(x_{i,j}) \phi(x_{i,j})^{\top} = \Phi_{i,t} S_{i,t} S_{i,t}^{\top} \Phi_{i,t}^{\top}.
\end{align*}

Furthermore, if $\tilde{H}_k$ denotes the subspace of $H_k$ spanned by the inducing points, then using $S_{i,t}$, we can also define the projection operator onto this subspace, which we denote by $P_{i,t}$. In particular, we have, 
\begin{align*}
    P_{i,t} = \Phi_{i,t} S_{i,t} (S_{i,t}^{\top} \Phi_{i,t}^{\top}\Phi_{i,t} S_{i,t})^{+} S_{i,t}^{\top} \Phi_{i,t}^{\top},
\end{align*}
where $A^{+}$ denotes the pseudo-inverse of the matrix $A$. Using $P_{i,t}$, we define an approximate feature map $\tilde{\phi}(\cdot) : P_{i,t} \phi( \cdot)$ that corresponds to the subspace $\tilde{H}_k$. Similar to $\Phi_{i,t}$, we define $\tilde{\Phi}_{i,t} = [\tilde{\phi}(x_{i,1}), \tilde{\phi}(x_{i,2}), \dots, \tilde{\phi}(x_{i,t})]^{\top} = \Phi_{i,t} P_{i,t}$. Lastly, we also define two more matrices, $A_{i,t} = \Phi_{i,t}^{\top} \Phi_{i,t} + \lambda I$ and $\tilde{A}_{i,t} = \tPhi_{i,t}^{\top} \tPhi_{i,t} + \lambda I$. \\

With the notations set up, we move on to proving Theorem~\ref{thm:sparse_fed_exp2}. The basic proof follows the same structure as that of Theorem~\ref{thm:regret_fed_exp}. We consider the regret separately in the exploration, communication and the exploitation sequence. For the exploration and the communication sequence, the regret is bounded with a constant factor of their lengths. For each client $i$, 
since the algorithm queries the same point, $\tilde{x}^*_i = \argmax_{x \in \cX} \tilde{\mu}_{N_T}^{(f_i)}(x)$ throughout the exploitation sequence, an upper bound on the regret incurred during the exploitation sequence is simply obtained by bounding the error $f_i(x^*) - f(\tilde{x}^*_{i,N_T})$ and multiplying with $T$, the length of the time horizon. The bound on $f_i(x^*_i) - f_i(\tilde{x}^*_{i, N_T})$ is obtained using Lemma~\ref{lemma:concentration_bound_sparse} using argument similar to the ones used in the proof of Theorem~\ref{thm:regret_fed_exp}. The overall regret is then obtained by summing the contributions of these two different terms. \\

We first consider the regret incurred during the exploration and communication phases. Let $R_1$ denote the regret incurred during the this period. We then have
\begin{align*}
    R_1 & = \sum_{i = 1}^K \sum_{t = 1}^{N_T + N_T^{(c)}} f_i(x_i^*) - f_i(x_{i,t}) \\
     & \leq \sum_{i = 1}^K \sum_{t = 1}^{N_T + N_T^{(c)}} 2\bar{B}_i \\
     & \leq \sum_{i = 1}^K 2\bar{B}_i (N_T + N_T^{c}) \\
     & \leq 2 \cdot (\max_i \bar{B}_i) \cdot  K (N_T + 9(1+ 1/\lambda)q_0 \gamma_{N_T}) \\
     & \leq 2 \cdot  (1 + 9(1+ 1/\lambda)q_0) \cdot  (\max_i \bar{B}_i) \cdot  K N_T.
\end{align*}
In the second step, we have once again used $\sup_{x \in \cX} h(x) \leq \|h\|_{H_k}$.

To bound the regret incurred during the exploitation sequence, denoted by $R_2$, we consider a discretization $\cD_{N_T}$ with $\cD_{N_T} = O(T^d)$ of the domain $\cX$. Based on Assumption~\ref{assumption:discretization}, we have that for all $x \in \cX$, $i \in \{1, 2, \dots, K\}$, $|h_i(x) - h_i([x]_{\cD_{N_T}})| \leq N_T^{-1/2}$ and $|\tilde{\mu}_{N_T}^{(h_i)}(x) - \tilde{\mu}_{N_T}^{(h_i)}([x]_{\cD_{N_T}})| \leq N_T^{-1/2}$, where $[x]_{\cD_{N_T}}$ denotes the point closest to $x$ on the discretization $\cD_{N_T}$. As before, we do not need to explicitly construct such a discretization for the actual algorithm. \\

Using Lemma~\ref{lemma:concentration_bound_sparse}, a union bound over all the points in the discretization $\cD_{N_T}$ and a series of argument similar to the ones used in the proof of Theorem~\ref{thm:regret_fed_exp} (See Appendix~\ref{proof:regret_fed_exp}), we can conclude that the following relations hold for all clients $i\in \{1,2,\dots, K\}$ for any given fixed $x' \in \cD_{N_T}$ with probability at least $1 - \delta_0/3$:
\begin{align*}
    |f_i(x') - \tilde{\mu}_{N_T}^{(f_i)}(x')| & \leq \alpha_i \cdot  \tilde{\beta}(B_i, \delta_0'') \cdot \tilde{\sigma}_{N_T}^{(h_i)}(x') + \frac{1 - \alpha_i}{K} \sum_{j = 1}^K \tilde{\beta}(B_j, \delta_0'')  \cdot \tilde{\sigma}_{N_T}^{(h_j)}(x')
\end{align*}
and, 
\begin{align*}
    f_i(x_i^*) - f_i(\tilde{x}^*_{i, N_T}) \leq \tilde{\sigma}_{N_T}^{(f_i)}(x_i^*) +  \tilde{\sigma}_{N_T}^{(f_i)}([\tilde{x}^*_{N_T}]_{\cD_{N_T}}) + 2N_T^{-1/2},
\end{align*}
where $\delta_0'' = \delta_0/(6K|\cD_{N_T}|)$,  $\displaystyle \tilde{\sigma}^{(f_i)}_{N_T}(x) := \alpha_i \cdot  \tilde{\beta}(B_i, \delta_0'') \cdot \tilde{\sigma}_{N_T}^{(h_i)}(x) + \frac{1 - \alpha_i}{K} \sum_{j = 1}^K \tilde{\beta}(B_j, \delta_0'')  \cdot \tilde{\sigma}_{N_T}^{(h_j)}(x)$ and $\tilde{x}^*_{i, N_T} = \argmax_{x \in \cX} \tilde{\mu}_{N_T}^{(f_i)}(x)$.


To bound $\tilde{\sigma}_{N_T}^{(f_i)}(\cdot)$, we first bound our \emph{approximate} predictive standard deviation $\tilde{\sigma}_{N_T}^{(h_j)}(\cdot)$ for all $j \in \{1,2,\dots, K\}$ based on our construction of the set of inducing points. For the particular choice of $q_0$ as used in Lemma~\ref{lemma:inducing_set_size}, the authors in~\cite{Calandriello2019} establish the following relation which holds with probability at least $1 - \delta_0/3$ for all clients $j \in \{1,2,\dots, K\}$: 
\begin{align*}
    \frac{1}{\chi} \left[\sigma_{N_T}^{(h_j)}(x) \right]^2 \leq \left[\tilde{\sigma}_{N_T}^{(h_j)}(x) \right]^2 \leq \chi \left[\sigma_{N_T}^{(h_j)}(x)\right]^2,
\end{align*}
where $\sigma_{N_T}^2(x)$ is the true predictive variance based on all the $N_T$ points queried during the exploration phase and $\chi = (1 + \varepsilon)/(1 - \varepsilon)$. This implies that \emph{approximate} posterior variance and the \emph{actual} posterior variance are within a constant factor of each other. On using this relation along with Lemma~\ref{lemma:posterior_std_dev_upper_bound}, we can conclude
\begin{align*}
    \tilde{\sigma}^{(f_i)}_{N_T}(x) & = \alpha_i \cdot  \tilde{\beta}(B_i, \delta_0'') \cdot \tilde{\sigma}_{N_T}^{(h_i)}(x) + \frac{1 - \alpha_i}{K} \sum_{j = 1}^K \tilde{\beta}(B_j, \delta_0'')  \cdot \tilde{\sigma}_{N_T}^{(h_j)}(x) \\
    & \leq \alpha_i \cdot  \tilde{\beta}(B_i, \delta_0'') \cdot \chi \cdot {\sigma}_{N_T}^{(h_i)}(x) + \frac{1 - \alpha_i}{K} \sum_{j = 1}^K \tilde{\beta}(B_j, \delta_0'')  \cdot \chi \cdot {\sigma}_{N_T}^{(h_j)}(x) \\
    & \leq \chi \sqrt{\frac{12\gamma_{N_T}}{N_T}} \cdot  \left(\alpha_i \cdot  \tilde{\beta}(B_i, \delta_0'')  + \frac{1 - \alpha_i}{K} \sum_{j = 1}^K \tilde{\beta}(B_j, \delta_0'') \right) \\
    & \leq \chi \sqrt{\frac{12\gamma_{N_T}}{N_T}} \cdot   \tilde{\beta}(\bar{B}_i, \delta_0'').
\end{align*}


Consequently, $R_2$ can be bounded as
\begin{align*}
    R_2 & = \sum_{i = 1}^K \sum_{t = N_T + N_T^{c} + 1}^T f_i(x_i^*) - f_i(\tilde{x}^*_{i, N_T}) \\
    & \leq T \sum_{i = 1}^K  \left( \tilde{\sigma}_{N_T}^{(f_i)}(x_i^*) +  \tilde{\sigma}_{N_T}^{(f_i)}([\tilde{x}^*_{N_T}]_{\cD_{N_T}}) + 2N_T^{-1/2} \right) \\
    & \leq 2KT  \left[ \tilde{\beta}(\bar{B}_i, \delta_0'') \sqrt{\frac{12 \gamma_{N_T}}{N_T}} + \frac{1}{\sqrt{N_T}} \right].
\end{align*}

Adding the bound on $R_1$ and $R_2$ yields the required result.

\subsection{Proof of Lemma~\ref{lemma:concentration_bound_sparse}}
\label{proof:concentration_bound_sparse}

Throughout the proof, we use the same notation as established earlier with the subscript corresponding to the client dropped for simplicity. We begin considering the definition of $\tilde{\mu}_t$. We have,
\begin{align*}
    h(x) - \tilde{\mu}_t(x) & = h(x) - k_{\bfz_t, x}^{\top} (\lambda K_{\bfz_t, \bfz_t} +  K_{\bfz_t, \bfx_t} K_{\bfx_t, \bfz_t})^{-1} K_{\bfz_t, \bfx_t} \bfy_t \\
    & = h(x) - k_{\bfz_t, x}^{\top} (\lambda K_{\bfz_t, \bfz_t} +  K_{\bfz_t, \bfx_t} K_{\bfx_t, \bfz_t})^{-1} K_{\bfz_t, \bfx_t} h_{1:t} - k_{\bfz_t, x}^{\top} (\lambda K_{\bfz_t, \bfz_t} +  K_{\bfz_t, \bfx_t} K_{\bfx_t, \bfz_t})^{-1} K_{\bfz_t, \bfx_t} {\epsilon}_{1:t},
\end{align*}
where $h_{1:t} = [h(x_1), h(x_2), \dots, h(x_t)]^{\top}$ and $\epsilon_{1:t} = [\epsilon_1, \epsilon_2, \dots, \epsilon_t]^{\top}$. \\

We focus on the first term which can be rewritten as
\begin{align*}
    |h(x) - k_{\bfz_t, x}^{\top} (\lambda K_{\bfz_t, \bfz_t} +  K_{\bfz_t, \bfx_t} K_{\bfx_t, \bfz_t})^{-1} K_{\bfz_t, \bfx_t} h_{1:t}| & = | \phi^{\top}(x) h - \phi^{\top}(x) \tilde{A}_t^{-1} \tPhi_t^{\top} h_{1:t}| \\ 
    & = | \phi^{\top}(x) h - \phi^{\top}(x) \tilde{A}_t^{-1} \tPhi_t^{\top} \Phi_t h| \\
    & \leq \left\| \phi^{\top}(x)\left(I - \tilde{A}_t^{-1} \tPhi_t^{\top} \Phi_t \right)  \right\|_{H_k} \|h\|_{H_k}.
\end{align*}

Consider, 
\begin{align*}
    \tilde{A}_t^{-1} \tPhi_t^{\top} \Phi_t  - I & = \tilde{A}_t^{-1}  \left(\tPhi_t^{\top} \Phi_t  - \tilde{A}_t \right) \\
    & = \tilde{A}_t^{-1} \left(\tPhi_t^{\top} \Phi_t  - (\tPhi_t^{\top} \tPhi_t + \lambda I) \right) \\
    & = \tilde{A}_t^{-1} \left(\tPhi_t^{\top} \Phi_t  - \tPhi_t^{\top} \Phi_t P_t  -  \lambda I \right) \\
    & = \tilde{A}_t^{-1} \tPhi_t^{\top} \Phi_t (I - P_t)   -  \lambda \tilde{A}_t^{-1} .
\end{align*}

Thus,
\begin{align*}
    \| \phi^{\top}(x)\left(I - \tilde{A}_t^{-1} \tPhi_t^{\top} \Phi_t \right)  \|_{H_k} & = \| \phi^{\top}(x) \left(\tilde{A}_t^{-1} \tPhi_t^{\top} \Phi_t (I - P_t)   -  \lambda I \right)  \|_{H_k} \\
    & \leq \| \phi^{\top}(x) \tilde{A}_t^{-1} \tPhi_t^{\top} \Phi_t (I - P_t)  \|_{H_k}  +  \lambda \|\phi^{\top}(x) \tilde{A}_t^{-1}  \|_{H_k} \\
    & \leq \sqrt{\phi^{\top}(x) \tilde{A}_t^{-1} \tPhi_t^{\top} \Phi_t (I - P_t)^2 \Phi_t^{\top} \tPhi_t \tilde{A}_t^{-1} \phi(x)}  +  \lambda \sqrt{\phi^{\top}(x) \tilde{A}_t^{-2} \phi(x)}  \\
    & \leq \sqrt{\phi^{\top}(x) \tilde{A}_t^{-1} \tPhi_t^{\top} \Phi_t (I - P_t) \Phi_t^{\top} \tPhi_t \tilde{A}_t^{-1} \phi(x)}  +  \lambda \sqrt{\phi^{\top}(x) \tilde{A}_t^{-2} \phi(x)}  \\
    & \leq \sqrt{\frac{\lambda}{1 - \varepsilon} \phi^{\top}(x) \tilde{A}_t^{-1} \tPhi_t^{\top} \tPhi_t \tilde{A}_t^{-1} \phi(x)}  +  \lambda \sqrt{\phi^{\top}(x) \tilde{A}_t^{-2} \phi(x)}  \\
    & \leq \sqrt{2 \left\{\frac{\lambda}{1 - \varepsilon} \left[\phi^{\top}(x) \tilde{A}_t^{-1} \phi(x) - \lambda\phi^{\top}(x) \tilde{A}_t^{-2} \phi(x) \right] + \lambda^2\phi^{\top}(x) \tilde{A}_t^{-2} \phi(x) \right\} }  \\
    & \leq \sqrt{2 \left\{\frac{\lambda}{1 - \varepsilon} \phi^{\top}(x) \tilde{A}_t^{-1} \phi(x)  + \lambda^2\phi^{\top}(x) \tilde{A}_t^{-2} \phi(x) \left(1 - \frac{1}{1 - \varepsilon} \right) \right\} }  \\
    & \leq \sqrt{ \left\{\frac{2\lambda}{1 - \varepsilon} \phi^{\top}(x) \tilde{A}_t^{-1} \phi(x) \right\} }  \\
    & \leq \sqrt{ \frac{2\lambda}{1 - \varepsilon} } \tilde{\sigma}_t(x).
\end{align*}
The fifth line follows from Lemma 1 in~\cite{Calandriello2018} and by noting that for the given choice of $q_0$, the dictionaries are $\varepsilon$-accurate with a probability of at least $1 - \delta_0/3$. On combining everything, we obtain,
\begin{align*}
    |h(x) - k_{\bfz_t, x}^{\top} (\lambda K_{\bfz_t, \bfz_t} +  K_{\bfz_t, \bfx_t} K_{\bfx_t, \bfz_t})^{-1} K_{\bfz_t, \bfx_t} h_{1:t}| \leq \| h \|_{H_k} \sqrt{ \frac{2\lambda}{1 - \varepsilon} } \tilde{\sigma}_t(x).
\end{align*}

We now consider the other term given by $k_{\bfz_t, x}^{\top} (\lambda K_{\bfz_t, \bfz_t} +  K_{\bfz_t, \bfx_t} K_{\bfx_t, \bfz_t})^{-1} K_{\bfz_t, \bfx_t} {\epsilon}_{1:t}$. For simplicity, we denote $k_{\bfz_t, x}^{\top} (\lambda K_{\bfz_t, \bfz_t} +  K_{\bfz_t, \bfx_t} K_{\bfx_t, \bfz_t})^{-1} K_{\bfz_t, \bfx_t} := \zeta^{\top}(x)$. Using the fact that the components of $\varepsilon_{1:t}$ are independent $R$-sub-Gaussian random variables, we have,
\begin{align*}
    \E \left[ \exp\left( \zeta^{\top}(x) {\varepsilon}_{1:t}  \right)\right] = \exp \left( \frac{R^2}{2} \| \zeta(x)\|^2\right).
\end{align*}

To bound the RHS, we focus on bounding the norm of the vector $\zeta(x)$. We have,
\begin{align*}
    \| \zeta(x) \|^2 & = \zeta^{\top}(x) \zeta(x) \\
    & = k_{\bfz_t, x}^{\top} (\lambda K_{\bfz_t, \bfz_t} +  K_{\bfz_t, \bfx_t} K_{\bfx_t, \bfz_t})^{-1} K_{\bfz_t, \bfx_t} K_{\bfx_t, \bfz_t} (\lambda K_{\bfz_t, \bfz_t} +  K_{\bfz_t, \bfx_t} K_{\bfx_t, \bfz_t})^{-1} k_{\bfz_t, x} \\
    & = k_{\bfz_t, x}^{\top} (\lambda K_{\bfz_t, \bfz_t} +  K_{\bfz_t, \bfx_t} K_{\bfx_t, \bfz_t})^{-1} (\lambda k_{ZZ} + K_{\bfz_t, \bfx_t} K_{\bfx_t, \bfz_t} - \lambda k_{ZZ}) (\lambda K_{\bfz_t, \bfz_t} +  K_{\bfz_t, \bfx_t} K_{\bfx_t, \bfz_t})^{-1} k_{\bfz_t, x} \\
    & = k_{\bfz_t, x}^{\top} (\lambda K_{\bfz_t, \bfz_t} +  K_{\bfz_t, \bfx_t} K_{\bfx_t, \bfz_t})^{-1} k_{\bfz_t, x}  - \lambda k_{\bfz_t, x}^{\top} (\lambda K_{\bfz_t, \bfz_t} +  K_{\bfz_t, \bfx_t} K_{\bfx_t, \bfz_t})^{-1} K_{\bfz_t, \bfz_t} (\lambda K_{\bfz_t, \bfz_t} +  K_{\bfz_t, \bfx_t} K_{\bfx_t, \bfz_t})^{-1} k_{\bfz_t, x} \\
    & \leq k_{\bfz_t, x}^{\top} (\lambda K_{\bfz_t, \bfz_t} +  K_{\bfz_t, \bfx_t} K_{\bfx_t, \bfz_t})^{-1} k_{\bfz_t, x} \\
    & \leq \frac{1}{\lambda} \left( k(x,x) -  k_{\bfz_t, x}^{\top} K_{\bfz_t, \bfz_t}^{-1} k_{\bfz_t, x} k_{\bfz_t, x}^{\top} ( K_{\bfz_t, \bfz_t} +  \lambda^{-1} K_{\bfz_t, \bfx_t} K_{\bfx_t, \bfz_t})^{-1} k_{\bfz_t, x}\right) \\
    & \leq \tilde{\sigma}_t^2(x).
\end{align*}

Thus, by using Chernoff-Hoeffding inequality, we can conclude that 
\begin{align*}
    |k_{\bfz_t, x}^{\top} (\lambda K_{\bfz_t, \bfz_t} +  K_{\bfz_t, \bfx_t} K_{\bfx_t, \bfz_t})^{-1} K_{\bfz_t, \bfx_t} {\epsilon}_{1:t}| \leq R \tilde{\sigma}_t(x) \sqrt{2 \log \left( \frac{1}{\delta} \right)}
\end{align*}
holds probability at least $1 - \delta$. On adding the two terms, we have the required answer.

\section{Proof of Theorem~\ref{Theorem:LB}}
\label{proof:LB}

To establish the lower bound on regret for collaborative learning with personalization, we focus on a simple system with two clients and a central server. The main idea in the proof is to construct a difficult instance for learning which enforces collaboration between clients with different objective functions. The construction of this instance of functions in the given RKHS is based on the ideas developed in~\cite{Scarlett2017}. We first describe some preliminaries of constructing functions of interest in the desired RKHS, which is adopted from the results in~\cite{Scarlett2017}, and then proceed with the proof of Theorem~\ref{Theorem:LB}. Throughout the remaining proof, we assume that the underlying RKHS corresponds to a given kernel belonging to the family of Squared Exponential and Matern kernels. \\


\subsection{Preliminaries on constructing functions in RKHS}

Consider the following multi-dimensional bump function defined over $\R^d$,
\begin{align*}
    \Phi(\xi) = \begin{cases} \exp \left( - \frac{1}{1- \| \xi\|^2} \right) & \text{ if } \|\xi\|^2 \leq 1 \\ 0 & \text{ otherwise }. \\  \end{cases}
\end{align*}
Let $\phi(x)$ be inverse Fourier transform of $\Phi(\xi)$. Since $\Phi$ is a function with finite energy, we can conclude that there exists a $\zeta > 0$ such that for $\|x\|_{\infty} > \zeta$, $\phi(x) \leq 0.5 \phi(0)$. Using $\phi(x)$, we define $\overline{\phi}(x;\theta)$ as the following function for any $\theta > 0$:
\begin{align*}
    \overline{\phi}(x; \theta, B) = \frac{\theta}{\phi(0)} \phi \left( \frac{x \zeta}{w_{\theta}} \right),
\end{align*}
where $w_{\theta}$ is a parameter that ensures that the RKHS norm of $\overline{\phi}$ is at most $B$. It is shown in~\cite{Scarlett2017} that for any given kernel belonging to the family of Squared Exponential and Matern kernels and constant $B > 0$, one can choose $w_{\theta}$, based on $\theta$, such that the corresponding RKHS norm of $\overline{\phi}$ is at most $B$ for all $\theta \leq \theta_0(B)$. In particular, for Square Exponential Kernel $\w_{\theta} = C \left( \log(B/\theta) \right)^{-1/2}$ and for Matern Kernel with smoothness parameter $\nu$, $w_{\theta} = C'(\theta/B)^{1/\nu})$ for some constants $C, C' > 0$. Here $\theta_0(B)$ is a threshold based on the kernel parameters and $B$. For exact expressions of these results, please refer to~\cite{Scarlett2017}. For simplicity of notation, we will only consider defining $\overline{\phi}$ for $\theta \leq \theta_0$ and implicitly assume the choice of $w_{\theta}$ is tuned to the particular $\theta$.  \\

Using $\overline{\phi}$, we define the following two functions:
\begin{align*}
    \varphi_0(x) := \overline{\phi}(x; \theta_1, B_0/2); \quad \varphi_1(x) := \overline{\phi}(x; \varepsilon, B_0/2).
\end{align*}
In the above definitions, $\theta_1 = 0.5 \cdot \theta_0(B_0/2)$, $B_0$ is the required upper bound on the observation functions and $\varepsilon > 0$ is a small constant whose value will be specified later. 

\subsection{Establishing the lower bound on regret}

With this basic tool for constructing RKHS functions, we are ready to build a difficult instance to establish our lower bound. For simplicity of exposition, we consider the case of $K = 2$ clients. The extension to the general case is straightforward which we briefly describe at the end of the proof. We consider the case of collaborative learning between two clients with observation functions $h_1$ and $h_2$ respectively and personalization parameters $\alpha_1$ and $\alpha_2$. We allow the pair $(\alpha_1, \alpha_2)$ to be any pair in $[0,1] \times [0,1]$ satisfying $\alpha_{\star} := \max\{ \min\{\alpha_1, 1 - \alpha_1\}, \min\{\alpha_2, 1 - \alpha_2\}\} > 0$. The objective functions of the two clients are given as $f_1 = \eta_1 h_1 + (1 - \eta_1)h_2$ and $f_2 = \eta_2 h_1 + (1 - \eta_2)h_2$, where $\eta_1 = (1 + \alpha_1)/2$ and $\eta_2 = (1 + \alpha_2)/2$. Consequently, $(\eta_1, \eta_2) \in \{(x,y) \in [0.5, 1]^2: \max\{x,y\} \geq 0.5 + \Delta, \min\{x,y\} \leq 1 - \Delta\}$ where $\Delta = \alpha_{\star}/2$. \\



We begin with considering the cases for which the pair $(\eta_1, \eta_2)$ satisfies the condition $\eta_1 \in [1/2, 1-\Delta]$ and $\eta_2 \in [1/2 + \Delta, 1]$. Note that the proof is identical for the pairs of $(\eta_1, \eta_2)$ satisfying $\eta_1 \in [1/2 + \Delta, 1]$ and $\eta_2 \in [1/2, 1 - \Delta]$ by simply exchanging the role of the clients $1$ and $2$. \\



To construct the instance of interest, consider the domain given as $\cX = \cX_0 \cup \cX_1$,
where $\cX_0 = [0,1]^d$ and $\cX_1$ is a shifted version of $\cX_0$, centered at $\bar{z} = (\upsilon_0 + 3/2,0,0,\dots,0)$. In the above definition, $\upsilon_0 = \inf \{\upsilon : \forall \ x \text{ s.t. } \|x\|_2 \geq \upsilon, \varphi_1(x) \leq 1/T^2\}$. The existence of such an $\upsilon_0$ is guaranteed by the fact that $\phi(x)$ decays to zero faster than any power of $\|x\|_2$. See Remark~\ref{remark:upsilon_choice} for more details.  \\

To define our observation functions, we consider a collection of $m_0 = \lfloor w_{\varepsilon}^{-d} \rfloor$ points, denoted by $\{z_1, z_2, \dots, z_{m_0}\}$, which forms a uniform grid over $\cX_0$. For a given pair $(\eta_1, \eta_2)$ we define our observation function $h_1$ and $h_2$ to be as follows:
\begin{align*}
    h_1(x) := -\frac{(1 - \eta_1)\Delta}{\eta_1 + \eta_2 - 1} \varphi_0(x - \bar{z}); \quad h_2(x) = \frac{\eta_1\Delta}{\eta_1 + \eta_2 - 1} \varphi_0(x - \bar{z}) + \varphi_1(x - z_M).
\end{align*}
In the above definition, $M$ is a random variable chosen uniformly over the set $\{1,2\dots, m_0\}$. Consequently, $h_2$ is a random function based on the value of $M$. It is not difficult to note that coefficients of $\varphi_0$ and $\varphi_1$ in both the functions are less than $1$ and hence the RKHS norm of both $h_1$ and $h_2$ is less than $B_0$, as required. For the observations, we assume that the noise is Gaussian with unit variance. \\

With this choice of observation functions, our objective functions are given as:
\begin{align*}
    f_1(x) :=  (1 - \eta_1)\varphi_1(x - z_M); \quad f_2(x) = \Delta \varphi_1(x) + \eta_2 \varphi_1(x - z_M).
\end{align*}
Note that client $2$ does not need to sample any point in $\cX_0$ to optimize its own objective, but would need to do so in order to help client $1$ optimize their objective function. This will form the crux of establishing the lower bound. \\

We introduce some additional notation for the analysis of the lower bound. Throughout the proof, a subscript $i$ will be used to refer to quantities related to client $i = 1,2$. Let $\{\cR_m\}_{m = 1}^{m_0}$ be partition of $\cX_0$ consisting of $m_0$ regions centred on the points $\{z_1, z_2, \dots, z_{m_0}\}$. Let $(x_{i,t}, y_{i,t})$ denote the input observation pair at time $t$ and consequently we define $j_{i,t} = \{ m : x_{i,t} \in \cR_m \}$. Using this, we can define $N_{i,m} = \sum_{t = 1}^T \mathbbm{1}\{ j_{i,t} = m \}$ for $m \in \{1,2,\dots, m_0\}$.  \\

We use $\bP_{m}( \cdot)$ to denote the distribution of the rewards when $M = m$ and $\bP_{*}(\cdot)$ to denote the joint distribution of the rewards and the random variable $M$. Let $\bP_0$ denote the reward distribution for the case when $h_2(x) = \dfrac{\eta_1\Delta}{\eta_1 + \eta_2 - 1} \varphi_0(x - \bar{z})$. 
We denote the corresponding expectations as $\E_{m}[\cdot], \E_{*}[\cdot]$, and $\E_{0}[\cdot]$.

Let $\pi = (\pi_1, \pi_2)$ denote the combined strategy for the two clients. For simplicity, we assume that the policy is deterministic. The arguments can be extended for the case of randomized strategies in a straightforward manner. Furthermore, since there is no constraint on communication, we assume that all the observations are exchanged between the two clients. This assumption is without loss of generality as any policy that communicates only a subset of observations can not do any better than the policy that communicates all the observations. Let $\cI_t = (y_{1,t}, y_{2,t})$ denote information gained at time $t$ and $\cI^{(t)} = (\cI_1, \cI_2, \dots, \cI_t)$ denote information state at time $t$. Since the policy is deterministic, $x_{1,t}$ and $x_{2,t}$ are deterministic functions of $\cI^{(t-1)}$. \\

We state the following lemma that will be useful in the analysis.
\begin{lemma}
Let $a$ be any function defined on the information state vector $\cI^{(T)}$ whose range is in the bounded interval $[0,A]$. Then for any $m \in \{1,2,\dots,m_0\}$, we have,
\begin{align*}
    \E_{\bP_m}[a(\cI^{(T)})] \leq \E_{\bP_0}[a(\cI^{(T)})] + A \sqrt{\frac{1}{2T^3} + \sum_{l = 1}^{m_0} \E_{0} [N_{2,l}] D_{m}^{l}},
\end{align*}
where $\displaystyle D_{m}^{l} = \max_{x \in \cR_l} \textsc{D}(\bP_{0}(  y | x) \| \bP_{m}( y | x) )$.
\label{lemma:kl_divergence}
\end{lemma}
The lemma is an adapted version of Lemma B.1 in~\cite{Auer1995} and Lemmas 3 and 4 in~\cite{Scarlett2017}. The proof of the lemma is provided at the end of the section for completeness. \\

We first focus on bounding the regret of the first client under the policy $\pi$ denoted by $R_{1, \pi}$. We fix the value of $M = m$. To lower bound the expected regret incurred by the first client, we upper bound the expected reward earned by the client under the policy $\pi$. 
\begin{align*}
    \sum_{t = 1}^{T} \E_{m}[f_1(x_{1,t})] & = \sum_{t = 1}^{T} \E_{m}\left[  \1\{x_{1,t} \in \cX_1\} f_1(x_{1,t}) + \sum_{l = 1}^{m_0} \1\{x_{1,t} \in \cR_l\} f_1(x_{1,t}) \right] \\
    & \leq \sum_{t = 1}^{T} \E_{m}\left[  \1\{x_{1,t} \in \cX_1\} \frac{1-\eta_1}{T^2} + (1 - \eta_1)\sum_{l = 1}^{m_0} \1\{x_{1,t} \in \cR_l\} \vartheta_{m}^{l} \right] \\
    & \leq  \frac{1 -\eta_1 }{T} + (1 - \eta_1)\sum_{l = 1}^{m_0} \sum_{t = 1}^{T} \E_{m}\left[ \1\{x_{1,t} \in \cR_l\}  \right] \vartheta_{m}^{l} \\
    & \leq  \frac{1- \eta_1}{T} + (1 - \eta_1)\sum_{l = 1}^{m_0} \E_m[N_{1,l}] \vartheta_{m}^{l}.
\end{align*}
In the above expression, $\vartheta_{m}^{l} = \max_{x \in \cR_l} \varphi_1(x - z_{m})$. On plugging in the result of Lemma~\ref{lemma:kl_divergence} for $N_{1,l}$ in the above expression, we obtain
\begin{align*}
    \sum_{t = 1}^{T} \E_{m}[f_1(x_{1,t})] & \leq  \frac{1-\eta_1}{T} + (1- \eta_1)\sum_{l = 1}^{m_0} \E_m[N_{1,l}] \vartheta_{m}^{l} \\
    & \leq  \frac{1-\eta_1}{T} + (1- \eta_1)\sum_{l = 1}^{m_0} \vartheta_{m}^{l} \left( \E_0[N_{1,l}] + T \sqrt{\frac{1}{2T^3} + \sum_{r = 1}^{m_0} \E_{0} [N_{2,r}] D_{m}^{r}} \right).
\end{align*}
Averaging over the choice of $M = m$, we obtain
\begin{align*}
    \frac{1}{m_0} \sum_{m = 1}^{m_0} \sum_{t = 1}^{T} \E_{m}[f_1(x_{1,t})] & \leq  \frac{1}{m_0} \sum_{m = 1}^{m_0} \frac{1-\eta_1}{T} + \frac{1 - \eta_1}{m_0} \sum_{m = 1}^{m_0} \sum_{l = 1}^{m_0} \vartheta_{m}^{l} \left( \E_0[N_{1,l}] + T \sqrt{\frac{1}{2T^3} + \sum_{r = 1}^{m_0} \E_{0} [N_{2,r}] D_{m}^{r}} \right)
\end{align*}
\begin{align*}
    \implies \sum_{t = 1}^{T} \E_{*}[f_1(x_{1,t})] & \leq  \frac{1-\eta_1}{T} + \frac{1-\eta_1}{m_0}  \sum_{l = 1}^{m_0} \left(\sum_{m = 1}^{m_0} \vartheta_{m}^{l}\right)  \E_0[N_{1,l}] +   \frac{T(1- \eta_1)}{m_0} \sum_{m = 1}^{m_0} \left( \sum_{l = 1}^{m_0} \vartheta_{m}^{l} \right) \sqrt{\frac{1}{2T^3} + \sum_{r = 1}^{m_0} \E_{0} [N_{2,r}] D_{m}^{r}}.
\end{align*}

Using the result from Lemma 4 in~\cite{Scarlett2017}, we can conclude that there exists constants $C_1, C_2, C_3 > 0$ such that
\begin{align*}
    \sum_{m = 1}^{m_0} \vartheta_{m}^{l} \leq C_1 \varepsilon; \quad  \sum_{l = 1}^{m_0} \vartheta_{m}^{l}  \leq C_2 \varepsilon; \quad \sum_{m = 1}^{m_0} D_{m}^{r} \leq C_3 \varepsilon^2.
\end{align*}
Using the above relations, we can rewrite the previous equation as
\begin{align*}
    \sum_{t = 1}^{T} \E_{*}[f_1(x_{1,t})] & \leq  \frac{1-\eta_1}{T} + \frac{1-\eta_1}{m_0}  \sum_{l = 1}^{m_0} \left(\sum_{m = 1}^{m_0} \vartheta_{m}^{l}\right)  \E_0[N_{1,l}] +   \frac{T(1 -\eta_1)}{m_0} \sum_{m = 1}^{m_0} \left( \sum_{l = 1}^{m_0} \vartheta_{m}^{l} \right) \sqrt{\frac{1}{2T^3} + \sum_{r = 1}^{m_0} \E_{0} [N_{2,r}] D_{m}^{r}} \\
    & \leq  \frac{1-\eta_1}{T} + \frac{C_1(1- \eta_1)\varepsilon}{m_0}  \sum_{l = 1}^{m_0} \E_0[N_{1,l}] +   \frac{C_2 T(1 - \eta_1) \varepsilon}{m_0} \sum_{m = 1}^{m_0} \sqrt{\frac{1}{2T^3} + \sum_{r = 1}^{m_0} \E_{0} [N_{2,r}] D_{m}^{r}} \\
    & \leq  \frac{1-\eta_1}{T} + \frac{C_1(1- \eta_1)\varepsilon T}{m_0}  +  C_2 T(1 - \eta_1) \varepsilon  \sqrt{ \frac{1}{m_0}  \sum_{m = 1}^{m_0} \left(\frac{1}{2T^3} + \sum_{r = 1}^{m_0} \E_{0} [N_{2,r}] D_{m}^{r} \right)} \\
    & \leq  \frac{1-\eta_1}{T} + \frac{C_1(1- \eta_1)\varepsilon T}{m_0}  +  C_2 T(1 - \eta_1) \varepsilon  \sqrt{  \frac{1}{2T^3} + \frac{1}{m_0}   \sum_{r = 1}^{m_0} \E_{0} [N_{2,r}] \left(\sum_{m = 1}^{m_0} D_{m}^{r} \right)} \\
    & \leq  \frac{1 -\eta_1}{T} + \frac{C_1(1- \eta_1)\varepsilon T}{m_0}  +  C_2 T(1 - \eta_1) \varepsilon  \sqrt{  \frac{1}{2T^3} + \frac{C_3 \varepsilon^2}{m_0}   \sum_{r = 1}^{m_0} \E_{0} [N_{2,r}] }.
\end{align*}
In the third step, we used the Jensen's inequality. \\

Using the upper bound on the reward obtained above, the regret can be lower bounded as
\begin{align*}
    \E_{*}[R_{1,\pi}] \geq (1- \eta_1)\varepsilon T \left(1 - \frac{1}{\varepsilon T^2} - \frac{C_1}{m_0}  -  C_2   \sqrt{  \frac{1}{2T^3} + \frac{C_3 \varepsilon^2 \overline{N}}{m_0}    }  \right),
\end{align*}
where $\displaystyle \overline{N} = \sum_{r = 1}^{m_0} \E_{0} [N_{2,r}]$. \\

We now focus on lower bounding the regret incurred by client $2$. Note that for sufficiently small $\varepsilon$, the maximizer for $f_2$ clearly lies in $\cX_1$. As a result, whenever client $2$ queries a point in $\cX_0$, it incurs a regret of at least $\Delta \theta_1 - \varepsilon \geq \Delta \theta_1/2$. Thus, for any $M = m$, the regret incurred by client $2$ can be lower bounded as
\begin{align*}
    \E_{m}[R_{2, \pi}] \geq \frac{\Delta \theta_1}{2} \sum_{r = 1}^{m_0} \E_{m} [N_{2,r}] \geq \frac{\Delta \theta_1}{2} \sum_{r = 1}^{m_0} \E_{0} [N_{2,r}].
\end{align*}
Consequently,
\begin{align*}
    \E_{*}[R_{2, \pi}] \geq \frac{\Delta \theta_1}{2} \sum_{r = 1}^{m_0} \E_{0} [N_{2,r}] = \frac{\Delta \theta_1 \overline{N}}{2}.
\end{align*}

Recall that $m_0$ is a function of $\varepsilon$ based on the kernel, making the lower bounds on regret of both the clients a function of $\varepsilon$. We choose a value of $\varepsilon$ based on $T$ to obtain a lower bound on the regret. Firstly, upon constraining $\varepsilon > 3/T^2$ and choosing a large enough value of $T$ (or equivalently, a small enough $\varepsilon$), we can ensure that $C_1/m_0 < 1/3$. Consequently, we have,
\begin{align*}
    \E_{*}[R_{1,\pi}] & \geq (1- \eta_1)\varepsilon T \left(\frac{1}{3}  -  C_2   \sqrt{  \frac{1}{2T^3} + \frac{C_3 \varepsilon^2 \overline{N}}{m_0}    }  \right), \\
    \E_{*}[R_{2, \pi}] & \geq \frac{\Delta \theta_1 \overline{N}}{2}.
\end{align*}

We set $\varepsilon = C_4 (m_0/\overline{N})^{1/2}$ for an appropriately chosen constant $C_4 > 0$, which yields $\E_{*}[R_{1,\pi}]  = \Omega(T(m_0/\overline{N})^{1/2})$ and $\E_{*}[R_{2, \pi}] = \Omega(\overline{N})$. Upon equating the two expressions to ensure the tightest bound, we obtain that $\overline{N} = \Theta(T^{2/3} m_0^{1/3}) \implies \varepsilon = \Theta((m_0/T)^{1/3})$. \\

Plugging in the relation between $m_0$ and $\varepsilon$ based on the kernel parameters, we arrive at result. In particular, for a Matern Kernel with smoothness parameter $\nu$, $m_0 = \Theta(\varepsilon^{-d/\nu})$, which gives us a values of $\varepsilon = \Theta(T^{-\nu/(3\nu + d)})$ and consequently a regret lower bound of $\Omega(\alpha_{\star} T^{(2\nu + d)/(3\nu + d)})$. Similarly, for the Squared Exponential Kernel, repeating a similar process yields a lower bound of $\Omega(\alpha_{\star} T^{2/3} (\log(T))^{d/6})$. The leading constant $\alpha_{\star}$ follows directly from the lower bounds $\E_{*}[R_{1,\pi}]$ and $\E_{*}[R_{2, \pi}]$ where $1 - \eta_1 \geq \Delta = \alpha_{\star}/2$.


\begin{remark}
\emph{Choosing the value of $\upsilon_0$}: Note that for the final choice of $\varepsilon$, the corresponding parameter $w_{\varepsilon} = \Theta(T^{-1/(3\nu + d)})$. Since the function $\phi(x)$ decays to zero faster than any power of $\|x\|_2$, there exists constants $C_5, C_6$ such that for all $x$ with $\|x\|_2 > C_5$, $\phi(x) \leq C_6/\|x\|_2^{2(3\nu + d)}$. Consequently, for any $\nu > C_5$, $\varphi_1(x) \leq 1/T^2$ due to the scaling with $w_{\varepsilon}$. Thus, one can guarantee $\upsilon_0$ to be finite and it can be chosen based on the decay of $\phi(x)$ and the kernel parameters.
\label{remark:upsilon_choice}
\end{remark}

\subsection{Extension to the case of \texorpdfstring{$K > 2$}{K > 2}}

For the case of $K > 2$, the approach is very similar to the case of $K = 2$. The only difference is in the construction of the observation and objective functions. WLOG, let $\alpha_1 \leq \alpha_2 \leq \dots \leq \alpha_K$. Similar to the case of $K = 2$, let $\eta_i = (1 + (K-1)\alpha_i)/K$ and $\displaystyle f_i = \eta_i h_i + \frac{1 - \eta_i}{K - 1}\sum_{k = 1, k \neq i}^K h_k$. The observation functions are $h_1, h_2, \dots, h_K$ are defined as follows:
\begin{align*}
    h_1 & = -\frac{(1 - \eta_1)\lambda_0}{K - 1}\varphi_0(x - \bar{z}); \\
    h_i & = -\frac{\eta_1\lambda_0}{K - 1}\varphi_0(x - \bar{z}); \quad \forall \ i =\{2,3,\dots, K - 1\}, \\
    h_K & = \eta_1 (K - 1) \lambda_0 \varphi_0(x - \bar{z}) + \varphi_1(x - z_M),
\end{align*}
where $\bar{z}$ and $z_M$ are defined as before and $\lambda_0 := \dfrac{\Delta'(K - 1)^2}{\eta_1\eta_K(K^3 - 2K^2) + \eta_K - \eta_1(K^2 - 3K + 1) - 1}$ with $\Delta' := (K - 1)\alpha_{\star}/K$. \\

With this choice of the observation functions, the objective function of client $i = \{1,2,\dots, K-1\}$ is given by
\begin{align*}
    f_i = \left( \frac{\eta_i + (2K -1)\eta_1 - K^2 \eta_i\eta_1 -1}{(K-1)^2}\right)\lambda_0 \varphi_0(x - \bar{z}) + \frac{1 - \eta_i}{K - 1}\varphi_1(x - z_M).
\end{align*}
Since $\eta_i \geq \eta_1$ (because $\alpha_i \geq \alpha_1$), $ \dfrac{\eta_i + (2K -1)\eta_1 - K^2 \eta_i\eta_1 -1}{(K-1)^2} \leq - \left( \dfrac{\eta_1 K - 1}{K - 1} \right)^2 \leq 0$. This implies that the maximizer of $f_i$ lies in $\cX_0$. On the other hand, the objective function of client $K$ is given as
\begin{align*}
    f_K = \Delta' \varphi_0(x - \bar{z}) + \frac{1 - \eta_K}{K - 1}\varphi_1(x - z_M).
\end{align*}
Similar to the case of $K = 2$, the maximizer of $f_K$ lies in $\cX_1$ but it would need to sample points in $\cX_0$ in order to help other clients. The rest of the argument follows the same way as for the case of $K = 2$ leading to the same conclusion.

\subsection{Proof of Lemma~\ref{lemma:kl_divergence}}
\label{proof:kl_divergence}

The proof of this lemma is similar to the lemma obtained for the lower bound in adversarial bandits in~\cite{Auer1995}. We have,
\begin{align*}
    \E_{m}[a(\cI^{(T)})] - \E_{0}[a(\cI^{(T)})] & = \sum_{\cI^{(T)}} a(\cI^{(T)}) \left(\bP_m(\cI^{(T)}) - \bP_{0}(\cI^{(T)})\right) \\
    & \leq \sum_{\cI^{(T)} :\bP_m(\cI^{(T)}) \geq \bP_{0}(\cI^{(T)}) } a(\cI^{(T)}) \left(\bP_m(\cI^{(T)}) - \bP_0(\cI^{(T)})\right) \\
    & \leq A \sum_{\cI^{(T)} :\bP_m(\cI^{(T)}) \geq \bP_{0}(\cI^{(T)}) }  \left(\bP_m(\cI^{(T)}) - \bP_{0}(\cI^{(T)})\right) \\
    & \leq \frac{A}{2}  \left\|\bP_m(\cI^{(T)}) - \bP_{0}(\cI^{(T)})\right\|_1 \\
    & \leq A \sqrt{ \textsc{D} \left( \bP_0(\cI^{(T)})\| \bP_{m}(\cI^{(T)})\right)}.
\end{align*}
In the last step, we have used Pinsker's inequality and $\textsc{D}(\bP \| \bQ)$ denotes the Kullback Leibler Divergence between two distributions $\bP$ and $\bQ$. \\

We can compute the KL divergence between $\bP_{0}$ and $\bP_{m}$ using the chain rule. We have,
\begin{align*}
    \textsc{D}(\bP_{0} \| \bP_{m} ) & = \sum_{t = 1}^T  \textsc{D}(\bP_{0}(\cI_t | \cI^{(t-1)}) \| \bP_{m}(\cI_t | \cI^{(t-1)}) )  \\
    & = \sum_{t = 1}^T \textsc{D}(\bP_{0}( (y_{1,t}, y_{2,t}) | \cI^{(t-1)}) \| \bP_{m}( (y_{1,t}, y_{2,t}) | \cI^{(t-1)}) ) \\
    & = \sum_{t = 1}^T \textsc{D}(\bP_{0}( y_{1,t} | \cI^{(t-1)}) \| \bP_{m}( y_{1,t} | \cI^{(t-1)}) ) + \textsc{D}(\bP_{0}( y_{2,t} | \cI^{(t-1)}) \| \bP_{m}( y_{1,t}  | \cI^{(t-1)}) ) \\
    & = \sum_{t = 1}^T \textsc{D}(\bP_{0}(  y_{2,t} | \cI^{(t-1)}) \| \bP_{m}( y_{2,t} | \cI^{(t-1)}) ) \\
    & = \sum_{t = 1}^T  \E_{0} \left[ \left(\mathbbm{1} \{ x_{2,t} \in \cX_1 \} + \sum_{l = 1}^{m_0} \mathbbm{1} \{ x_{2,t} \in \cR_l \} \right)  \textsc{D}(\bP_{0}(  y_{2,t} | x_{2,t}) \| \bP_{m}( y_{2,t} | x_{2,t}) ) \right].
\end{align*}
The third step uses the independence of observations across clients (conditioned on the query points) and the fourth step follows by noting that reward distribution of client $1$ is the same under both $\bP_0$ and $\bP_m$. Since $\bP_0$ and $\bP_m$ are Gaussian distributions with unit variance, their KL divergence is half the square of difference of their means. Using the definition of $\upsilon_0$, we can conclude that $\textsc{D}(\bP_{0}(  y_{2,t} | x_{2,t}) \| \bP_{m}( y_{2,t} | x_{2,t}) ) \leq 1/(2T^4)$ for all $x_{2,t} \in \cX_1$. Using the definition of $D_{m}^{l}$ for the regions in $\cX_0$, we can rewrite the previous equation as,
\begin{align*}
    \textsc{D}(\bP_{0} \| \bP_{m} ) & = \sum_{t = 1}^T  \E_{0} \left[ \left(\mathbbm{1} \{ x_{2,t} \in \cX_1 \} + \sum_{l = 1}^{m_0} \mathbbm{1} \{ x_{2,t} \in \cR_l \} \right)  \textsc{D}(\bP_{0}(  y_{2,t} | x_{2,t}) \| \bP_{m}( y_{2,t} | x_{2,t}) ) \right] \\
    & \leq \sum_{t = 1}^T  \E_{0} \left[ \mathbbm{1} \{ x_{2,t} \in \cX_1 \} \cdot \frac{1}{2T^2} + \sum_{l = 1}^{m_0} \mathbbm{1} \{ x_{2,t} \in \cR_l \} D_{m}^{l} \right] \\
    & \leq T \cdot \frac{1}{2T^4} + \sum_{l = 1}^{m_0}    \E_{0} \left[\sum_{t = 1}^T \mathbbm{1} \{ x_{2,t} \in \cR_l \}  \right] D_{m}^{l} \\
    & \leq \frac{1}{2T^3} + \sum_{l = 1}^{m_0} \E_{0} [N_{2,l}] D_{m}^{l}.
\end{align*}

\section{Supplemental Material on Experiments}
\label{sec:experiments}



In this section, we provide further details on the experiments. Branin is a standard benchmark function for Bayesian Optimization and its analytical expression is given below~\cite{Azimi2012, Picheny2013}
 $$\displaystyle B(x, y) = -\frac{1}{51.95}\left( \left( v - \frac{5.1u^2}{4 \pi^2} + \frac{5u}{\pi} - 6\right)^2 + \left(10 - \frac{10}{8\pi} \right)\cos(u) - 44.81\right),$$ where $u = 15x - 5$ and $v = 15y$. We have also used the function from~\cite{sobester2008engineering} which is given by
 $$F(x,y) = (6x - 2)^2 \sin(12x- 4) + (6y - 2)^2 \sin(12y- 4).$$

The domain is set to $\mathcal{X} = [0,1]^2$. The details of IGP-UCB, PI and EI are provided below. 

\begin{enumerate}
    \item IGP-UCB: The algorithm is implemented exactly as outlined in~\cite{Chowdhury2017} with $B$ (in scaling parameter $\beta_t$) set to $15$ . The parameters $R$ and $\delta_0$ are set to $10^{-2}$ and $10^{-3}$. $\gamma_t$ was set to $\log t$. 
    \item Expected Improvement(EI)/Probability of Improvement (PI): Similar to IGP-UCB, EI and PI select the observation points based on maximizing an index often referred to as an acquisition function. The acquisition function of EI is $(\mu(x) - f^* - \varepsilon)\Phi\left( z \right) + \sigma(x)\phi(z)$, where $z = \frac{\mu(x) - f^* - \varepsilon}{\sigma(x)}$. The acquisition function of PI is $\Phi(z)$, where $z = \frac{\mu(x) - f^* - \xi}{\sigma(x)}$. Here, $\Phi(\cdot)$ and $\phi(\cdot)$ denote the CDF and PDF of a standard normal random variable. $f^*$ is set to be the maximum value of $\mu(x)$ among the current observation points. The parameters $\varepsilon$ and $\xi$ are used to balance the exploration-exploitation trade-off. We follow~\cite{hoffman2011portfolio} that showed the best choice of these parameters are small non-zero values. In particular, $\varepsilon$ and $\xi$ are both set to $0.01$.
\end{enumerate}

\end{document}